\documentclass[11pt]{article}
\usepackage{amssymb}

\usepackage{booktabs} 

\usepackage{multirow}

\usepackage{algorithm}
\usepackage{algorithmic}

\usepackage[table]{xcolor}

\usepackage{bm}

\newcommand{\defeq}{\coloneqq}

\usepackage{complexity}

\let\R\relax
\let\E\relax

\newcommand{\hvx}{\hat{\vx}}

\newcommand{\hvy}{\hat{\vy}}

\newcommand{\contr}{\rho}

\newcommand{\supp}{\mathrm{supp}}

\usepackage{hyperref}

\newcommand{\declarecolor}[2]{\definecolor{#1}{RGB}{#2}\expandafter\newcommand\csname #1\endcsname[1]{\textcolor{#1}{##1}}}
\declarecolor{White}{255, 255, 255}
\declarecolor{Black}{0, 0, 0}
\declarecolor{Maroon}{128, 0, 0}
\declarecolor{Coral}{255, 127, 80}
\declarecolor{Red}{182, 21, 21}
\declarecolor{LimeGreen}{50, 205, 50}
\declarecolor{DarkGreen}{0, 80, 0}
\declarecolor{Purple}{146, 42, 158}
\declarecolor{Navy}{0, 0, 128}
\declarecolor{LightBlue}{84, 101, 202}
\definecolor{mydarkblue}{rgb}{0,0.08,0.45}
\hypersetup{ %
    pdftitle={},
    pdfkeywords={},
    pdfborder=0 0 0,
    pdfpagemode=UseNone,
    colorlinks=true,
    linkcolor=Navy,
    citecolor=DarkGreen,
    filecolor=Purple,
    urlcolor=Black,
}

\usepackage{fullpage}
\usepackage{authblk}

\usepackage{MnSymbol}

\usepackage{amsmath,amsfonts,bm}
\usepackage{multirow,array}
\usepackage{diagbox}

\def\va{{\bm{a}}}
\def\vb{{\bm{b}}}
\def\vc{{\bm{c}}}

\def\ve{{\bm{e}}}

\def\vp{{\bm{p}}}

\def\vx{{\bm{x}}}
\def\vy{{\bm{y}}}
\def\vz{{\bm{z}}}

\DeclareMathAlphabet{\mathsfit}{\encodingdefault}{\sfdefault}{m}{sl}
\SetMathAlphabet{\mathsfit}{bold}{\encodingdefault}{\sfdefault}{bx}{n}

\newcommand{\calD}{\ensuremath{\mathcal{D}}}

\newcommand{\calX}{\ensuremath{\mathcal{X}}}

\newcommand{\calZ}{\ensuremath{\mathcal{Z}}}

\newcommand{\vxstar}{\vx^*}

\newcommand{\E}{\mathbb{E}}

\newcommand{\R}{\mathbb{R}}

\renewcommand{\bar}[1]{\overline{#1}}

\newcommand{\norm}[1]{\left\| #1 \right\|}

\DeclareMathOperator*{\argmax}{argmax}
\DeclareMathOperator*{\argmin}{argmin}

\usepackage{amsthm}
\usepackage{mathtools}
\usepackage{amsmath}
\usepackage{bbm}
\usepackage{amsfonts}

\newcommand{\TwoDSperner}{{\textsc{2D-Sperner}} }

\newcommand{\epsilonThickBrou}{{\textsc{$\epsilon$-ThickBrouwer}} }

\newcommand{\localmaxcut}{{\textsc{LocalMaxCut}} }

\newcommand{\maxcut}{{\textsc{MaxCut}} }

\usepackage{nicefrac}
\usepackage{esvect}

\usepackage{tikz}
\usetikzlibrary{positioning, backgrounds, fit, arrows.meta, calc, decorations.pathreplacing, shapes.arrows, shadows, shapes, backgrounds, intersections, patterns.meta, quotes, angles}

\definecolor{uiGreenLight}{HTML}{E8F5E9}  
\definecolor{uiGreenMedium}{HTML}{C8E6C9} 
\definecolor{uiGreenDark}{HTML}{81C784}   
\definecolor{uiRedLight}{HTML}{FFEBEE}    
\definecolor{uiRedDark}{HTML}{E57373}     
\definecolor{uiText}{HTML}{263238}        

\usepackage{subfigure}

\usepackage[%
linewidth=2pt,
linecolor=gray,
middlelinecolor= black,
middlelinewidth=0.4pt,
roundcorner=1pt,
topline = false,
rightline = false,
bottomline = false,
rightmargin=0pt,
skipabove=0pt,
skipbelow=0pt,
leftmargin=0pt,
innerleftmargin=4pt,
innerrightmargin=0pt,
innertopmargin=0pt,
innerbottommargin=0pt,
]{mdframed}

\usepackage[capitalize,noabbrev,nameinlink]{cleveref}

\usepackage{enumitem}
\usepackage[dvipsnames]{xcolor}

\usepackage{bbm}
\newcommand{\bbm}[1]{\mathbbm{#1}}

\usepackage{natbib}
\renewcommand{\vec}[1]{\bm{#1}}
\newcommand{\mat}[1]{\mathbf{#1}}

\usepackage{thm-restate}

\theoremstyle{plain}
\newtheorem{theorem}{Theorem}[section]
\newtheorem{lemma}[theorem]{Lemma}
\newtheorem{corollary}[theorem]{Corollary}
\newtheorem{proposition}[theorem]{Proposition}

\newtheorem{claim}[theorem]{Claim}
\newtheorem{assumption}[theorem]{Assumption}

\theoremstyle{definition}
\newtheorem{definition}[theorem]{Definition}

\theoremstyle{remark}
\newtheorem{remark}[theorem]{Remark}
\newtheorem{example}[theorem]{Example}

\title{On the Computational Complexity of Performative Prediction}

\author[1]{Ioannis Anagnostides\thanks{These authors contributed equally.}}
\author[2]{Rohan Chauhan}
\author[2]{Ioannis Panageas}
\author[1,3]{Tuomas Sandholm}
\author[2]{Jingming Yan\protect\footnotemark[1]}

\affil[1]{Carnegie Mellon University}
\affil[2]{University of California, Irvine}
\affil[3]{\small Additional affiliations: Strategy Robot, Inc., Strategic Machine, Inc., Optimized Markets, Inc.}
\affil[ ]{\texttt{\{ianagnos,sandholm\}}\texttt{@cs.cmu.edu}, \texttt{\{ipanagea,rmchauha,jingmy1\}}\texttt{@uci.edu}}

\begin{document}

\maketitle

\begin{abstract}
    Performative prediction captures the phenomenon where deploying a predictive model shifts the underlying data distribution. While simple retraining dynamics are known to converge linearly when the performative effects are weak ($\rho < 1$), the complexity in the regime $\rho > 1$ was hitherto open. In this paper, we establish a sharp phase transition: computing an $\epsilon$-performatively stable point is \PPAD-complete---and thus polynomial-time equivalent to Nash equilibria in general-sum games---even when $\rho = 1 + O(\epsilon)$. This intractability persists even in the ostensibly simple setting with a quadratic loss function and linear distribution shifts. One of our key technical contributions is to extend this \PPAD-hardness result to general convex domains, which is of broader interest in the complexity of variational inequalities. Finally, we address the special case of strategic classification, showing that computing a strategic local optimum is \PLS-hard.
\end{abstract}

\section{Introduction}

Machine learning models are typically developed under the assumption that the deployment environment is static: the underlying data distribution remains fixed regardless of the model's predictions. However, in many high-stakes social and economic domains, this premise is fundamentally flawed. As sociologists and economists have long observed, models are not merely cameras that passively record the markets, but engines that actively shape the reality they aim to model~\citep{Mackenzie08:Engine}. Similarly, in modern predictive tasks---ranging from credit scoring to spam filtering---the deployment of a predictive model triggers a shift in the underlying data distribution, as agents react strategically to the deployed classifier.

This ubiquitous phenomenon was formalized by~\citet{perdomo2021performativeprediction} as \emph{performative prediction}. In this setting, the model, parameterized by $\vx \in \calX$, induces a distribution $\calD(\vx)$ over the data. The decision-maker is thus facing a moving target: updating the model triggers a shift in the very objective they seek to minimize. This feedback loop yields two natural solution concepts. First, a \emph{performatively optimal} point minimizes the expected loss over the distribution it induces, $\E_{\vz \sim \calD(\vx)} [ \ell(\vx; \vz)]$ (\Cref{def:perf-opt}). In contrast, a \emph{performatively stable} point is minimizing expected loss for the \emph{fixed} distribution it induces (\Cref{def:perf-stab}).

Stability is a key desideratum, ensuring that the model remains invariant under retraining. Perhaps the most natural algorithmic approach to finding such points is by \emph{repeatedly} solving the risk minimization problem---a process coined \emph{repeated risk minimization (RRM)} by~\citet{perdomo2021performativeprediction}---until a fixed point is reached.

\begin{definition}[\citealp{perdomo2021performativeprediction}]
    \label{def:RRM}
   \emph{Repeated risk minimization (RRM)} refers to the procedure whereby, starting from an initial model parameterized by $\vx_0$, the following sequence of updates is performed.
   \begin{equation*}
       \vx_{t+1} = G(\vx_t) = \argmin_{\vx \in \calX} \E_{\vz \sim \calD(\vx_t)} [\ell(\vx; \vz)].
   \end{equation*}
\end{definition}

\citet{perdomo2021performativeprediction} showed that RRM is bound to converge linearly to a performatively stable point when $\contr \defeq L \beta/\alpha < 1$. Intuitively, this condition requires that the sensitivity of the distribution shift $L$ is small relative to the geometry of the loss landscape, governed by its smoothness $\beta$ and strong convexity $\alpha$ (\Cref{sec:prels} contains the precise definitions). On the other hand, \citet{perdomo2021performativeprediction} observed that RRM can fail to converge even when $ \contr = 1$, which means that performative effects are marginally stronger.

One might hope to circumvent this by simply increasing the regularization strength $\alpha$ to force the condition $\contr < 1$. However, such additional regularization can be destructive, potentially eliminating the meaningful equilibria.

Despite the significant progress in establishing improved convergence guarantees (\emph{e.g.},~\citealp{Khorsandi24:Tight,Mofakhami23:Performative}), the complexity of computing performatively stable points remains poorly understood when $\contr \geq 1$. In particular, a fundamental question arises:
\begin{quote}
    \centering
    \emph{Is the failure of RRM simply a limitation of specific retraining dynamics, or is identifying a performatively stable point intrinsically intractable in the presence of stronger performative effects?}
\end{quote}
The failure of RRM---and other algorithms such as repeated gradient descent and performative gradient descent~\citep{Izzo21:Learn}---when $\contr \geq 1$ does not by itself imply intractability. Indeed, as we show, there are efficient algorithms even when $\contr$ is slightly above 1.
    
\subsection{Our results}

We characterize the computational complexity of performative stability across the spectrum of $\contr$. To begin with, we establish the following result.

\begin{theorem}
    \label{theorem:main1}
    For any small enough $\epsilon > 0$, computing an $\epsilon$-performatively stable point for some $\contr = L \beta / \alpha \leq 1 + O(\epsilon)$ is \PPAD-hard.
\end{theorem}
This means that computing performatively stable points is as hard as finding Nash equilibria in general-sum games~\citep{Daskalakis09:The,Chen09:Settling}, which is unlikely to admit efficient algorithms. There is a basic trade off in~\Cref{theorem:main1} worth highlighting: \PPAD-hardness persists even if one is content with a crude approximation $\epsilon = \Theta(1)$, but that only precludes instances in which $L \beta / \alpha$ is some additive constant larger than $1$. At the other end of the spectrum, \PPAD-hardness kicks in even when $L \beta / \alpha - 1$ is exponentially small, as long as the desired precision is also small enough. 

We also show that an $\epsilon$-performatively stable point can be computed in $\poly(d, \log(1/\epsilon))$ time when $\contr = 1 + O_\epsilon(\epsilon^4)$ (\Cref{prop:exp-perf}). This improves upon repeated risk minimization (RRM) and other natural algorithms. While RRM converges for $\contr < 1$, the number of iterations scales with $\log^{-1} (\nicefrac{1}{\contr}) \approx \nicefrac{1}{1-\contr}$ when $\contr \approx 1$, which blows up when $\contr$ approaches 1. Moreover, a recent result by~\citet{Diakonikolas25:Pushing} implies a $\poly(1/\epsilon)$ algorithm when $\contr \leq 1 + O_\epsilon(\epsilon)$, matching~\Cref{theorem:main1} in the order of $\epsilon$. As a result, we find the transition from \PPAD-hardness to tractability to be particularly acute.

Furthermore, we establish unconditional, information-theoretic lower bounds, showing that any algorithm requires exponentially many ERM evaluations to find a performatively stable point (\Cref{cor:query}). This holds whether one uses repeated risk minimization or any other more sophisticated algorithm.

From a technical standpoint, our hardness results are established through simple, canonical reductions that encode any variational inequality or fixed point problem as an instance of performative stability (\Cref{theorem:VIred,theorem:FPred}). In particular, we show that intractability persists even in the ostensibly simple setting comprising a quadratic loss and an affine distribution shift (\Cref{theorem:PPAD_complete}). 


\paragraph{The nonexpansive regime} We go on to generalize the setup of performative prediction to general norms, extending the $\ell_2$ contraction argument of~\citet{perdomo2021performativeprediction} (\Cref{sec:generalnorms}). Interestingly, we observe that in this generalized setting, intractability barriers emerge \emph{even in the contractive regime}. Specifically, we show that finding performatively stable points would imply a complexity theory breakthrough (\Cref{prop:SSG}).

\paragraph{The role of the constraint set} An important component of our reduction is the geometry of the domain. Existing hardness results for variational inequalities and fixed points typically rely on the hypercube $\calX = [0, 1]^d$. However, this does not always mesh well with machine learning applications; for example, in the context of performative prediction, training a classifier constrained on the $\ell_2$ ball instead is perhaps more natural~\citep{hinton2012improving,Goodfellow-et-al-2016}. Surprisingly, the complexity of VIs and fixed points over general constraint sets has received limited attention, with some exceptions (\Cref{sec:related}). We fill this gap by showing that \PPAD-hardness persists under any reasonable convex constraint set (\Cref{theorem:PPAD_hardness_convex_set}).

\paragraph{Strategic classification} Finally, we turn to \emph{strategic classification}~\citep{Hardt16:Strategic}, which falls within the scope of performative prediction. We show that finding a \emph{local} optimum of the performative risk---under single-label updates---is \PLS-hard (\Cref{theorem:PLS_hardness}); $\PLS$ captures the complexity of (presumably) hard local optimization problems such as local max-cut. This complements the \NP-hardness of~\citet{Hardt16:Strategic} concerning \emph{global} performative optimality, and further highlights the intractability of attaining performative optimality. It shows that local search heuristics---often employed to sidestep \NP-hardness---can fail to efficiently identify stable points.
\subsection{Related work}
\label{sec:related}

Following the foundational work of~\citet{perdomo2021performativeprediction}, significant attention has been devoted to the convergence of retraining dynamics. \citet{Mendler-Dunner20:Stochastic} analyzed stochastic variants of RRM, distinguishing between ``greedy'' and ``lazy'' deployment. In their terminology, greedy deployment releases the new model at every step, whereas lazy deployment accumulates multiple gradient updates before releasing a new model. \citet{Zrnic21:Who} further refined those dynamics by studying two-timescale algorithms, showing that separating the timescales of model updates and the resulting distribution shifts can stabilize learning. \citet{Miller21:Echo} and~\citet{Izzo21:Learn} developed derivative-free methods to optimize the performative risk, targeting optimality rather than just stability. For more recent pointers, we refer to \citet{Khorsandi24:Tight,Mofakhami23:Performative}, and references therein. There has been some work addressing misspecification in the underlying distribution map~\citep{Xue24:Distributionally}, but our paper focuses on the standard model.

\paragraph{Strategic classification} Performative prediction encompasses the framework of \emph{strategic classification}~\citep{Hardt16:Strategic,Chen18:Strategyproof,Chen20:Learning,Dong18:Strategic}, where the distribution shift arises from individual agents rationally best-responding to the classifier. The performative prediction framework abstracts the game-theoretic interaction into the distribution map.

For a comprehensive overview of the emerging field of performative prediction, we refer to~\citet{Hardt25:Performative}. Additional related work appears in~\Cref{sec:furtherrelated}.


\section{Preliminaries}
\label{sec:prels}

\paragraph{Notation} For $\vx, \vx' \in \R^d$, we use $\langle \vx, \vx' \rangle$ for their inner product. $\|\vx \|_2 = \sqrt{ \langle \vx, \vx \rangle } $ denotes the Euclidean norm. $\|\cdot\|$ denotes an arbitrary norm. $\calX$ is a convex and compact subset of $\R^d$ that represents the parameter space of the decision-maker. For $\vx \in \calX$, $\calD(\vx)$ denotes the distribution induced by $\vx$. We will use the notation $\calZ \defeq \bigcup_{\vx \in \calX} \supp(\calD(\vx))$. For the sake of exposition, we sometimes write $O_\epsilon(\cdot)$ to denote the dependence only on the parameter $\epsilon$.

Performative prediction centers on the problem
\begin{align} \label{eq:optimization_objective}
    \min_{\vx \in \calX} \mathbb{E}_{\vz \sim \calD(\vx)}\left[ \ell(\vx; \vz)\right].
\end{align}
We now formally define performative optimality and performative stability.

\begin{definition}[Performative optimality; \citealp{perdomo2021performativeprediction}]
    \label{def:perf-opt}
    A point $\vxstar \in \calX$ is \emph{performatively optimal} if 
    \begin{align*}
        \vxstar \in \argmin_{\vx \in \calX} \mathbb{E}_{\vz \sim \calD(\vx)}\left[\ell(\vx; \vz)\right].
    \end{align*}
\end{definition}

Performatively optimal points correspond to \emph{Stackelberg equilibria}~\citep{Conitzer06:Computing,VonStackelberg34:Marktform}, as the decision-maker commits to a model, anticipating how the distribution $\calD$ will shift in response. In~\Cref{sec:strat-class}, we introduce a local version of~\Cref{def:perf-opt} in the context of strategic classification.

\begin{definition}[Performative stability; \citealp{perdomo2021performativeprediction}]
    \label{def:perf-stab}
    A point $\vxstar \in \calX$ is \emph{performatively stable} if
    \begin{align*}
        \vxstar \in \argmin_{\vx \in \calX} \mathbb{E}_{\vz \sim \calD(\vxstar)}\left[\ell(\vx; \vz)\right].
    \end{align*}
\end{definition}

Performatively stable points exist under mild assumptions~\citep{perdomo2021performativeprediction}. They are in correspondence to \emph{Nash equilibria}, as the decision-maker selects a model that is optimal for the \emph{current} distribution. Our complexity results leverage this connection to shed light on the complexity of performative prediction.

The following assumptions are made concerning the loss function and the magnitude of the distribution shift.

\begin{assumption}
    \label{ass:l2}
    Let $\ell(\vx; \vz)$ be the loss function and $\calD(\vx)$ the distribution on $\calZ$ induced by $\vx \in \calX$.
    \begin{itemize}[noitemsep]
        \item (strong convexity) $\ell(\vx; \vz)$ is $\alpha$-strongly convex with respect to $\|\cdot\|_2$:
        \begin{equation*}
            \ell(\vx; \vz) \ge \ell(\vx' ; \vz) + \langle \nabla_{\vx} \ell(\vx'; \vz), \vx - \vx' \rangle + \frac{\alpha}{2} \|\vx - \vx' \|_2^2
        \end{equation*}
        for any $\vx, \vx' \in \calX$ and $\vz \in \calZ$.
        \item (smoothness) $\ell(\vx; \vz)$ is $\beta$(-jointly) smooth if
        \begin{equation*}
            \|\nabla_{\vx} \ell(\vx; \vz) - \nabla_\vx \ell(\vx'; \vz) \|_2 \leq \beta \|\vx - \vx' \|_2
        \end{equation*}
        and
        \begin{equation*}
            \|\nabla_{\vx} \ell(\vx; \vz) - \nabla_\vx \ell(\vx; \vz') \|_2 \leq \beta \|\vz - \vz' \|_2
        \end{equation*}
        for any $\vx, \vx' \in \calX$ and $\vz, \vz' \in \calZ$.
        \item (sensitivity) $\calD$ is $L$-sensitive if
        \begin{equation*}
            W_1(\calD(\vx), \calD(\vx')) \leq L \|\vx - \vx' \|_2
        \end{equation*}
        for any $\vx, \vx' \in \calX$, where $W_1$ denotes the Wasserstein-1 distance, or earth mover's distance.
    \end{itemize}
\end{assumption}
We define $\contr \defeq L \beta/\alpha$. In~\Cref{sec:generalnorms}, we also generalize the setup of~\Cref{ass:l2} to general norms.

We rely on the following notion of approximation.

\begin{definition}
    \label{def:approx-perfstab}
    A point $\vxstar \in \calX$ is (first-order) \emph{$\epsilon$-performatively stable} if
\begin{equation*}   
    \left\langle\vx - \vxstar, \mathbb{E}_{\vz \sim \calD(\vxstar)}\left[\nabla_{\vx} \ell(\vxstar; \vz)\right]\right\rangle \geq - \epsilon \quad \forall \vx \in \calX.
\end{equation*}
\end{definition}

When the loss function is convex, the definition above coincides with~\Cref{def:perf-stab} (\Cref{claim:first-order-stable-equivalence}). In applications where the loss function is nonconvex~\citep{Li24:Stochastic}, \Cref{def:approx-perfstab} is the natural local relaxation of~\Cref{def:perf-stab}. Since computing (first-order) performatively stable points lies in $\PPAD$ (\Cref{cor:PPAD-memb}), our hardness result establishes \PPAD-completeness for that notion.

Another natural way to measure the approximation error is through the fixed point gap $\|\vxstar - G(\vxstar) \|_2$, where $G$ is the RRM map (\Cref{def:RRM}); as we formalize in~\Cref{lemma:direc1,lemma:direc2}, those notions are polynomially related.
 
\section{Complexity of performatively stable points}
\label{sec:complexity}

In this section, we characterize the complexity of performatively stable points. 

\paragraph{A hard class of problems} We consider the following class of performative prediction instances.
\begin{align} 
    & \min_{\vx \in \calX} \left\{ \ell(\vx; \vz) \defeq \frac{1}{2} \| \vx \|_2^2 - \vx^\top \vz \right\}, \label{eq:strongly-convex-objective}\\
    & \text{where } \vz = g(\vx).\label{eq:g-def}
\end{align}
We assume that $g$ is $L$-Lipschitz continuous, so that $\|g(\vx) - g(\vx') \|_2 \leq L \|\vx - \vx' \|_2$ for any $\vx, \vx' \in \calX$. The function $\ell$ defined in~\eqref{eq:strongly-convex-objective} is $1$-strongly convex in $\vx$ and $1$-jointly smooth (per~\Cref{ass:l2}), while the sensitivity of $\calD(\vx)$ is $L$. So, $\contr = L$ in this class.

The underlying distribution above is a singleton supported on $g(\vx)$. (The Wasserstein-1 distance between two point mass distributions is simply the distance between the two points.) Our reductions work more broadly for any distribution $\calD(\vx)$ such that $\E_{ \vz \sim \calD(\vx)} \vz = g(\vx)$; this could be, for example, a more well-behaved Gaussian distribution. This holds because $\E_{\vz \sim \calD(\vx)}[ \ell(\vx; \vz) ] = \E_{\vz \sim \calD(\vx)} [ \frac{1}{2} \|\vx \|^2 - \vx^\top \vz ] = \frac{1}{2} \|\vx \|^2 - \vx^\top \E_{\vz \sim \calD(\vx)} [\vz]$, by definition of the loss $\ell$ in~\eqref{eq:optimization_objective}. In other words, the choice of distribution does not alleviate the hardness of the problem.

\subsection{Encoding VIs and fixed points}

We now show how a suitable choice of $g$ allows us to encode hard optimization problems. First, we consider a \emph{variational inequality (VI)} problem given by a mapping $F : \calX \to \R^d$. An $\epsilon$-approximate VI solution is a point $\vxstar \in \calX$ such that $\langle \vx - \vxstar, F(\vxstar) \rangle \geq - \epsilon$ for all $\vx \in \calX$. We observe that by selecting $g: \vx \mapsto \vx - F(\vx)$, an $\epsilon$-performatively stable point of~\eqref{eq:strongly-convex-objective}-\eqref{eq:g-def} yields an $\epsilon$-approximate solution to the VI problem. Furthermore, considering $g : \vx \mapsto \vx - \bar{F}(\vx)$ for a damped (rescaled) mapping $\bar{F}$ makes the Lipschitz constant of $g$---the sensitivity of the distribution---approach 1 while rescaling the approximation factors between the two problems. We summarize this guarantee below.

\begin{proposition}[From VIs to performative stability]
    \label{theorem:VIred}
    For any $\epsilon > 0$ and $\epsilon' > 0$,
    computing an $\epsilon'$-approximate VI solution of an $L$-Lipschitz mapping $F$ reduces to computing $\epsilon$-performatively stable points with $\contr \leq 1 + \frac{\epsilon}{\epsilon'} L$.
\end{proposition}

We instantiate and sharpen this reduction in~\Cref{theorem:PPAD_complete} for the class of affine VI problems. First, we provide a similar reduction for fixed point problems. Here, we are given a continuous function $T : \calX \to \calX$ and the problem is to find an $\epsilon$-fixed point thereof. We observe that by selecting $g: \vx \mapsto (1 - \lambda) \vx + \lambda T(\vx)$ for $\lambda = \epsilon/\epsilon'$, we arrive at the following theorem.

\begin{proposition}[From fixed points to performative stability]
    \label{theorem:FPred}
    For any $\epsilon > 0$, computing an $\epsilon'$-fixed point of an $L$-Lipschitz continuous mapping reduces to computing an $\epsilon$-fixed point of the RRM map $G$ (\Cref{def:RRM}) with $\contr \leq 1 + \frac{\epsilon}{\epsilon'} L$.
\end{proposition}

It is well-known that fixed points can be reduced to VIs and vice versa, but the reduction above is particularly direct. We will use it to prove query lower bounds through the result of~\citet{Hirsch89:Exponential}.

We now leverage~\Cref{theorem:VIred} to establish \PPAD-hardness for a particularly simple class of problems: the distribution shift is given by an affine function and the constraint set $\calX$ is the hypercube. We begin by extracting a useful result from~\citet[Theorem 4.4]{bernasconi2024role} concerning the complexity of affine VIs on the hypercube. A closely related result was shown by~\citet{Rubinstein15:Inapproximability} in the context of computing approximate Nash equilibria in polymatrix, binary-action games. For our purposes, it is convenient to use the lemma as given by~\citet{bernasconi2024role} because of the assumed matrix bounds. For a matrix $\mat{A} \in \mathbb{R}^{d \times d}$, we denote by $\|\mat{A}\|_1$ its maximum absolute column sum and by $\|\mat{A}\|_\infty$ its maximum absolute row sum. Our proof uses the fact that the spectral norm $\|\mat{A} \|_2$ satisfies the inequality $\|\mat{A} \|_2 \leq \sqrt{\|\mat{A}\|_1 \|\mat{A}\|_\infty }$.

\begin{lemma}[\citealp{bernasconi2024role}] \label{lemma:inapproximation}
    It is \PPAD-complete to find a point $\vxstar \in [0, 1]^d$ such that for all $\vx \in [0, 1]^d$,
    \begin{align*}
        \left\langle\vx - \vx^*, \mat{A}\vx^* + \vb\right\rangle \geq - \epsilon'.
    \end{align*}
    This holds even when $\epsilon' > 0$ is an absolute constant, $\norm{\mat{A}}_1 \leq 1$, and $\norm{\mat{A}}_\infty \leq 1$.
\end{lemma}

This allows us to strengthen~\Cref{theorem:VIred} by establishing \PPAD-hardness with sharp constants and for a seemingly simple class of performative prediction instances.

\begin{restatable}{theorem}{PPADaffine} \label{theorem:PPAD_complete}
    Finding an $\epsilon$-performatively stable point per~\cref{def:approx-perfstab} is \PPAD-hard even when $L \beta / \alpha \leq 1 + \frac{\epsilon}{\epsilon'}$ for $\epsilon' = 0.088/6 \approx 0.0147$. This is so even when $\ell$ is a quadratic objective, $\ell(\vx; \vz) = \frac{1}{2} \| \vx - \vz \|_2^2$, and $\calD(\vx)$ is given by an affine map.
\end{restatable}

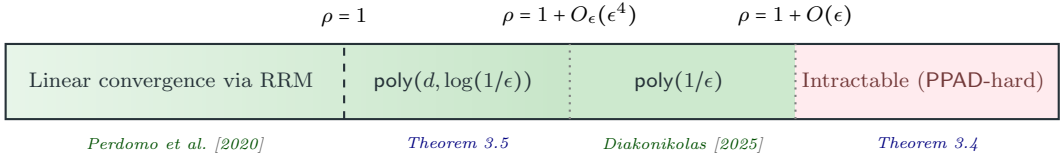
\begin{figure*}[t]
    \centering
    \begin{tikzpicture}[font=\sffamily, >=stealth]

    \def\barHeight{1.0}
    \def\barWidth{14}
    
    \def\xOne{4.5}    
    \def\xTwo{7.5}    
    \def\xThree{10.5} 
    \def\xEnd{14}


    \shadedraw[left color=uiGreenLight, right color=uiGreenMedium, draw=none] 
        (0,0) rectangle (\xTwo, \barHeight);
    
    \node[anchor=center, align=center, text=uiText, font=\scriptsize] at (\xOne/2, \barHeight/2) {
        Linear convergence via RRM
    };
    \node[anchor=north, text=black!60, font=\tiny\itshape, yshift=-3pt] at (\xOne/2, 0) {\citet{perdomo2021performativeprediction}};

    \node[anchor=center, align=center, text=uiText, font=\scriptsize] at ({(\xOne+\xTwo)/2}, \barHeight/2) {
        $\poly(d, \log(1/\epsilon))$
    };
    \node[anchor=north, text=black!60, font=\tiny\itshape, yshift=-3pt] at ({(\xOne+\xTwo)/2}, 0) {\Cref{prop:exp-perf}};

    \fill[uiGreenDark!40] (\xTwo,0) rectangle (\xThree, \barHeight); 
    \node[anchor=center, align=center, text=uiText, font=\scriptsize] at ({(\xTwo+\xThree)/2}, \barHeight/2) {
        $\poly(1/\epsilon)$
    };
    \node[anchor=north, text=black!60, font=\tiny\itshape, yshift=-3pt] at ({(\xTwo+\xThree)/2}, 0) {\citet{Diakonikolas25:Pushing}};

    \fill[uiRedLight] (\xThree,0) rectangle (\xEnd, \barHeight);
    \node[anchor=center, align=center, text=uiRedDark!50!black, font=\scriptsize] at ({(\xThree+\xEnd)/2}, \barHeight/2) {
        Intractable (\PPAD-hard)
    };
    \node[anchor=north, text=black!60, font=\tiny\itshape, yshift=-3pt] at ({(\xThree+\xEnd)/2}, 0) {\Cref{theorem:PPAD_complete}};

    
    \draw[thick, uiText] (0,0) rectangle (\xEnd, \barHeight);
    
    \draw[thick, dashed, uiText] (\xOne, 0) -- (\xOne, \barHeight);
    \node[above=3pt, font=\scriptsize] at (\xOne, \barHeight) {$\rho = 1$};
    
    \draw[dotted, thick, black!50] (\xTwo, 0) -- (\xTwo, \barHeight);
    \node[above=3pt, font=\scriptsize] at (\xTwo, \barHeight) {$\rho = 1 + O_\epsilon(\epsilon^4)$};
    
    \draw[dotted, thick, black!50] (\xThree, 0) -- (\xThree, \barHeight);
    \node[above=3pt, font=\scriptsize] at (\xThree, \barHeight) {$\rho = 1 + O(\epsilon)$};
    
    \node[right=5pt, font=\small] at (\xEnd, \barHeight/2) {};

\end{tikzpicture}
    \caption{The complexity landscape for computing $\epsilon$-performatively stable points.}
    \label{fig:phasetransition}
\end{figure*}

The simple proof is deferred to~\Cref{sec:proofs}. The constant $\epsilon'$ appearing above is obtained by combining~\Cref{lemma:inapproximation} with the inapproximability of~\citet{Deligkas24:Pure}. It is worth noting that for general VIs on the hypercube (without the restriction to affine mappings), an approximation of even $\approx \frac{1}{2}$ is \PPAD-hard~\citep{Deligkas23:Tight}.

To put~\Cref{theorem:PPAD_complete} into context, \citet{Diakonikolas25:Pushing} recently analyzed a relaxation of nonexpansiveness, showing that an $\epsilon$-fixed point---which yields an $O_\epsilon(\epsilon)$-performatively stable point on account of~\Cref{lemma:direc2}---can be computed in $\poly(1/\epsilon)$ time even when the Lipschitz constant $L$ of the map satisfies $L \leq 1 + \epsilon/ D$. This means that, subject to a complexity collapse, the bound on $\contr$ in~\Cref{theorem:PPAD_complete} cannot be improved up to the factor multiplying $\epsilon$. Taken together, we identify an acute phase transition in the complexity of the problem.

Furthermore, we also establish the following result.

\begin{restatable}{theorem}{expans}
    \label{prop:exp-perf}
    If $\contr \leq 1 + \epsilon$ (per~\Cref{ass:l2}), there is a $\poly(d, \log(1/\epsilon))$-time algorithm for computing an $O_{\epsilon}(\epsilon^{1/4})$-performatively stable point.
\end{restatable}
The approximation we obtain holds even in terms of the RRM fixed-point gap. For the natural VI counterpart, the approximation would be $O_\epsilon(\sqrt{\epsilon})$, and even $O_\epsilon(\epsilon)$ in terms of the \emph{Minty} VI solution (\Cref{prop:hypomonotone}).

\Cref{prop:exp-perf} complements the result of~\citet{Diakonikolas25:Pushing}: even though we need $\contr \leq 1 + O_{\epsilon}(\epsilon^4)$ to get an $\epsilon$-performatively stable point, \Cref{prop:exp-perf} scales logarithmically with $1/\epsilon$, which is an exponential improvement. On the other hand, the result of~\citet{Diakonikolas25:Pushing} is not confined to finite-dimensional Euclidean spaces. For an illustration of the different regions in terms of $\rho$ and their complexity, we refer to~\Cref{fig:phasetransition}.

\Cref{prop:exp-perf} is established by showing how to apply the ellipsoid algorithm on a VI problem in which the mapping $F$ is only approximately monotone---specifically, \emph{hypomonotone} (\Cref{prop:hypomonotone}). Our observation is that, in such problems, \emph{expected} variational inequalities in the sense of~\citet{Zhang25:Expected} induce approximate VI solutions, as we formalize in~\Cref{sec:ellipsoid}.

An interesting open question is whether~\Cref{prop:exp-perf} can be improved to match the approximation-expansiveness tradeoff established by~\citet{Diakonikolas25:Pushing}.

\paragraph{Unconditional lower bounds} We next establish unconditional query complexity hardness results. In the performative prediction setting, the natural query model we consider allows an algorithm to specify a point $\vx \in \calX$ and receive the induced RRM iterate $G(\vx)$. We will use the following seminal lower bound due to~\citet{Hirsch89:Exponential}.

\begin{theorem}[\citealp{Hirsch89:Exponential}]
    For any $d \geq 3$, any algorithm that finds an $\epsilon$-fixed point of a Lipschitz map $T : \calX \to \calX$ requires at least $c (( \frac{1}{\epsilon} - 10) L )^{d-2}$ steps, where $c$ is an absolute constant and $L$ is the Lipschitz constant of $T(\vx) - \vx$.
\end{theorem}
For the class of problems given in~\eqref{eq:strongly-convex-objective}-\eqref{eq:g-def}, it follows that $G(\vx) = g(\vx)$. As a result, under the reduction of~\Cref{theorem:FPred}, every ERM query outputs $(1 - \lambda) \vx + \lambda T(\vx)$, which reveals as much information as $T(\vx)$ itself.

\begin{corollary}
    \label{cor:query}
    Computing an $\epsilon$-fixed point of the RRM map $G$ (\Cref{def:RRM}) even when $\contr = L \beta / \alpha \leq 1 + O_{\epsilon}(\epsilon)$ requires $2^{\Omega(d)}$ ERM queries. This holds even when $\epsilon$ is a constant.
\end{corollary}

\subsection{General norms}
\label{sec:generalnorms}

Having characterized the complexity spectrum in the Euclidean setting (\Cref{ass:l2}), we turn to the more general setting. We extend~\Cref{ass:l2} to general norms (\Cref{ass:norms}), and show that the contraction analysis of~\citet{perdomo2021performativeprediction} carries over in this setting.\footnote{With an abuse of notation we use the same symbols to denote the analogous parameters even though convexity, smoothness, and sensitivity are now measured differently. The underlying choice of norm will be clear from the context.}

\begin{restatable}{proposition}{contractgen}
    \label{prop:contract-gen}
    If $L \beta / \alpha < 1$ per~\Cref{ass:norms}, the RRM map $G$ (\Cref{def:RRM}) is a contraction with respect to the norm $\|\cdot\|$. In particular, if $\vxstar$ is the unique fixed point,
    \begin{equation*}
        \|\vx_t - \vxstar \| \leq \frac{L \beta}{\alpha} \|\vx_{t-1} - \vxstar \| \leq \left( \frac{L \beta}{\alpha} \right)^t \|\vx_0 - \vxstar \|.
    \end{equation*}
\end{restatable}

The question now is to characterize the complexity of performatively stable points in this more general setting. The obvious algorithm that arises from~\Cref{prop:contract-gen} computes an $\epsilon$-performatively stable point in a number of iterations that grows as $\log(1/\epsilon) ( 1 - L \beta/\alpha)^{-1}$. When $1 - L \beta / \alpha \approx 0$, this can be prohibitive.

We first note that, in the regime where $\epsilon$ is not too small, this can be improved using \emph{Halpern iteration}~\citep{Halpern67:Fixed,Lieder21:Convergence,Wittmann92:Approximation,Diakonikolas20:Halpern}.

\begin{corollary}
    If $L \beta / \alpha \leq 1$ per~\Cref{ass:norms}, there is an algorithm that finds an $\epsilon$-performatively stable point and has complexity linear in $1/\epsilon$.
\end{corollary}

This is another setting in which RRM is inferior to alternative algorithms. The question that remains concerns the complexity when $\epsilon$ is exponentially small. We observe that this is as hard as a major open problem in complexity theory~\citep{Condon92:Complexity,etessami20:Tarski}.

\begin{proposition}
    \label{prop:SSG}
    Computing an $\epsilon$-performatively stable point even when the RRM map $G$ (\Cref{def:RRM}) satisfies $\| G (\vx) - G(\vx') \| < \|\vx - \vx' \|$ is as hard as solving a simple stochastic game (SSG).
\end{proposition}
For this hardness result, it suffices to consider the $\ell_\infty$ norm, and follows from the fact that the SSG problem reduces to finding fixed points of contractions in the $\ell_\infty$ norm. Unlike our previous results, the precondition of~\Cref{prop:SSG} is in terms of the Lipschitz constant of $G$, and not the upper bound $L \beta/\alpha$ (per~\Cref{ass:norms}).

\subsection{Complexity for general convex domains}

\begin{figure*}
    \centering
    \scalebox{0.5}{\input{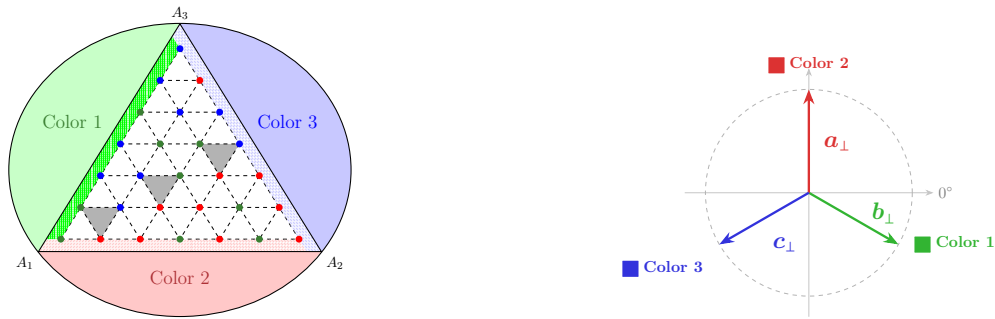}}
    \hspace{3cm}
    \scalebox{0.55}{\definecolor{colorOne}{RGB}{220, 50, 50}   
\definecolor{colorTwo}{RGB}{50, 180, 50}   
\definecolor{colorThree}{RGB}{50, 50, 220} 

\begin{tikzpicture}[scale=2.5, >=Stealth, thick]
    \coordinate (O) at (0,0);
    \coordinate (XAxis) at (1.5,0); 

    \draw[gray!45, thin, ->] (-1.2,0) -- (1.2,0) node[right, text=gray] {$0^\circ$};
    \draw[gray!45, thin, ->] (0,-1.2) -- (0,1.2);
    \draw[gray!60, thin, dashed] (O) circle (1cm);

    \coordinate (V1) at (-30:1);
    \draw[->, colorTwo, ultra thick] (O) -- (V1);
    \node[anchor=west, colorTwo] at ($(V1) + (0.1,0)$) 
         {\fcolorbox{colorTwo}{colorTwo}{\rule{0pt}{4pt}\rule{4pt}{0pt}} \textbf{Color 1}};
    \node[colorTwo, scale=1.5] at (-15:0.75) {$\vb_{\perp}$};

    \coordinate (V2) at (90:1);
    \draw[->, colorOne, ultra thick] (O) -- (V2);
    \node[anchor=south, colorOne] at ($(V2) + (0,0.1)$) 
         {\fcolorbox{colorOne}{colorOne}{\rule{0pt}{4pt}\rule{4pt}{0pt}} \textbf{Color 2}};
    \node[colorOne, scale=1.5] at (60:0.55) {$\va_{\perp}$};

    \coordinate (V3) at (210:1);
    \draw[->, colorThree, ultra thick] (O) -- (V3);
    \node[anchor=north east, colorThree] at ($(V3) + (-0.1,-0.1)$) 
         {\fcolorbox{colorThree}{colorThree}{\rule{0pt}{4pt}\rule{4pt}{0pt}} \textbf{Color 3}};
    \node[colorThree, scale=1.5] at (245:0.55) {$\vc_{\perp}$};


\end{tikzpicture}}
        \caption{Left: Illustration of the proof of \cref{theorem:PPAD_hardness_convex_set}. Regions outside $\triangle {A}_1 {A}_2 {A}_3$ are mapped with corresponding colors. The shaded bands represents the $\epsilon$-thickness strips used in the construction. Inside $\triangle {A}_1 {A}_2 {A}_3$ the domain forms a triangular grid with colors assigned to the vertices. Trichromatic triangles correspond to VI solutions, and are highlighted by shading. Right: Directional vector mapping induced by the coloring. The directions are chosen so that any point outside $\triangle {A}_1 {A}_2 {A}_3$ is pushed toward the interior of the triangle. As a result, all VI solutions must lie inside $\triangle {A}_1 {A}_2 {A}_3$.}
        \label{fig:PPAD_color}
\end{figure*}

While~\cref{theorem:PPAD_complete} characterizes the complexity of computing performatively stable points, it relies on prior results established with respect to a hypercube constraint set (that is, $[0, 1]^d$). However, this particular constraint set may not always be aligned with practical applications. For example, in modern machine learning it is often more common to impose an upper bound on the $\ell_2$ norm of the parameter $\vx$ \citep{Goodfellow-et-al-2016, hinton2012improving}, which translates to an $\ell_2$ ball constraint set.

To close this gap in the literature, we extend our hardness result to general convex sets under the mild assumption that the domain is \emph{well bounded}; this is a standard regularity condition in convex optimization~\citep{Grotschel93:Geometric}.

\begin{definition}[Well-bounded domains] \label{def:well-bounded}
    A convex and compact set $\mathcal{X}$ is called \emph{well bounded} if there exist $R_1 > 0$ and $R_2 > 0$ such that $\mathcal{B}_{R_1}(\mathbf{0}) \subseteq \mathcal{X} \subseteq \mathcal{B}_{R_2}(\mathbf{0}),$ where $\mathcal{B}_R(\mathbf{0})$ is the Euclidean ball centered at $\mathbf{0}$ with radius $R$.
\end{definition}

The well-bounded assumption is general and captures many common constraint sets of interest, including the hypercube and the $\ell_2$ ball. Our technical approach only requires that $\mathcal{X} \supseteq \mathcal{B}_{R_1}$ with respect to a two-dimensional ball, so our analysis can encompass sets that may not be fully dimensional, as is common in machine learning. Furthermore, without loss of generality, the center of the balls can be set at the origin by shifting the domain $\cal{X}$. 

We establish a complexity hardness result for solving variational inequalities (VIs) over well-bounded domains.

\begin{restatable}{theorem}{convexsethardness}
    \label{theorem:PPAD_hardness_convex_set}
    Given a convex and compact domain $\mathcal{X} \subset \mathbb{R}^d$ that is well bounded, an $L$-Lipschitz function $F: \mathcal{X} \to \mathbb{R}^d$, and $\epsilon = O(2^{-n})$, it is \PPAD-hard to find a point $\vx^* \in \mathcal{X}$ such that
    \begin{equation} \label{eq:target_VI}
        \left\langle\vx - \vx^*, F(\vx^*)\right\rangle \leq \epsilon \quad \forall \vx \in \mathcal{X}.
    \end{equation}
    This holds even when $d = 2$ and $L = O(1)$.
\end{restatable}

This result is of broader interest in the complexity of variational inequalities. The requirement that $\epsilon$ is exponentially small is necessary to prove hardness in low dimensions. In stark contrast, high-precision solutions can be attained in the contractive regime since RRM exhibits linear convergence.

\Cref{theorem:PPAD_hardness_convex_set} follows from constructing a polynomial-time reduction from the \TwoDSperner problem, which was shown to be \PPAD-complete by \citet{CHEN20094448}. Our proof proceeds by showing that, given a well-bounded domain $\cal{X},$ one can locate an equilateral triangle $\triangle {A}_1 {A}_2 {A}_3 \subseteq \cal{X}$, which will serve as the domain for the \TwoDSperner instance. To construct a continuous mapping $F:\cal{X} \to \cal{X}$ from the coloring of the \TwoDSperner problem, we carefully design an arithmetic circuit that converts the coloring of a given point to vectors in Euclidean space. A key technical challenge is that the arithmetic circuit $F$ is continuous, whereas the coloring in the \TwoDSperner problem is specified by a boolean circuit and is therefore discontinuous. As a result, $F$ cannot exactly match the coloring everywhere in $\cal{X}.$ We address this issue by employing a sampling technique introduced by \citet{deligkas_et_al:LIPIcs.ICALP.2020.38} which allows us to control the error arising from this mismatch. \cref{fig:PPAD_color} illustrates the coloring and the choice of directional vectors respectively. We defer the detailed proof to \cref{sec:append_convex_set_PPAD}.

We next apply our general result to the problem of computing performatively stable points. We show that, even under a general constraint set, \PPAD-hardness persists for some accuracy $\epsilon = O(2^{-n})$ when $\rho > 1.$ We summarize this result below.

\begin{corollary}
    For any convex compact domain $\cal{X}$ that is well-bounded, finding an $\epsilon$-performatively stable point per~\cref{def:perf-stab} is \PPAD-hard even when $L \beta / \alpha \leq 1 + \frac{\epsilon}{\epsilon'}$ for some $\epsilon' = O(2^{-n})$. This is so even when $\ell$ is a quadratic objective, $\ell(\vx; \vz) = \frac{1}{2} \| \vx - \vz \|_2^2 $.
\end{corollary}
The proof is similar to~\Cref{theorem:PPAD_complete}, using $g(\vx) = \vx + \frac{\epsilon}{\epsilon'}F(\vx)$, where $F(\vx)$ is the operator in \cref{theorem:PPAD_hardness_convex_set}.

\begin{remark}
    By choosing $\epsilon < \epsilon' = O(2^{-n}),$ the expansion parameter $L \beta / \alpha$ can be made arbitrarily close to $1$, while \PPAD-hardness still persists.
\end{remark}
\section{Strategic classification}
\label{sec:strat-class}

As highlighted in~\Cref{sec:related}, performative prediction encompasses the problem of strategic classification. The complexity of computing performatively optimal points---also known as ``strategic maxima'' in this line of work---was already shown to be \NP-hard in the original paper by~\citet{Hardt16:Strategic}. Two natural question arise: i) what is the complexity of \emph{local} performative optimality? And ii) what is the complexity of performative stability in strategic classification? Concerning the second question, the class of problems we have considered so far is no longer suitable because strategic classification is a more structured problem. Before we proceed, let us first recall the basic definition of strategic classification. 

\begin{definition}[Strategic classification; \citealp{Hardt16:Strategic}]
    \label{def:strat-class}
    Strategic classification is a game played betwen the \emph{Jury} and the \emph{Contestant}. Let $\calD$ be a distribution over a population $X$, $c : X \times X \to \R_{\geq 0}$ a cost function, and $h$ a target classifier.
    \begin{enumerate}
        \item The Jury first publishes a classifier $f : X \to \{0, 1\}$, which may depend on the cost function $c$, the distribution $\calD$, and the target classifier $h$.
        \item The Contestant, who knows $c$, $h$, $\calD$, and $f$, selects a deviation $\Delta: X \to X$.
    \end{enumerate}
    The payoff to the Jury is $\Pr_{\vx \sim \calD} [ h(\vx) = f(\Delta(\vx))]$ and the payoff to the Contestant is $\E_{\vx \sim \calD} [ f(\Delta(\vx)) - c(\vx, \Delta(\vx))]$. 
\end{definition}
Strategic classification is commonly formulated as a Stackelberg game between the Jury and the Contestant, in which the Contestant always best-responds to the classifier published by the Jury, while the Jury seeks to maximize their utility while accounting for the strategic deviation of the Contestant. The strategic maximum is defined as follows.

\begin{definition}[Strategic maximum; \citealp{Hardt16:Strategic}]
    Give a population $X$, a distribution $\cal{D}$ over the population, a cost function $c: X \times X \to \mathbb{R}_{\geq 0}$, and a target classifier $h$, a classifier $f^*$ for the Jury is said to be at strategic maximum if
    \begin{equation*}
        f^* \in \argmax_{f : X \to \{0, 1\}} \Pr_{\vx \sim \calD} [ h(\vx) = f(\Delta(\vx))].
    \end{equation*}
\end{definition}

\paragraph{Strategic local optimality}
Finding a global optimum in strategic classification has a wide range of applications,  but is computationally intractable. In particular, \citet{Hardt16:Strategic} established \NP-completeness. A natural question is to consider local search algorithms for the Jury, who can explore local moves into nearby configurations with the goal to converge to a local optimum. 

We begin by formally defining the notion of local moves for the Jury. Consider a finite population $X$ with $|X| = n,$ and a classifier $f$. We say that the Jury makes a local move to a nearby configuration by updating the label of a single data point. Formally, a classifier $f'$ is a nearby configuration of $f$ if there exists an index $i \in [n]$ such that
\begin{equation*}
    f'(X) = f(X) \oplus \ve_i,
\end{equation*}
where $\oplus$ denotes the XOR operation and $\ve_i$ is the $i$th standard basis vector. We denote by $\mathcal{N}(f)$ the set of all classifiers that can be reached from $f$ with one local move. We now define the notion of strategic local optimum.
\begin{definition}\label{def:strat_local_opt}
    Given a finite population $X$, a distribution $\cal{D}$ over $X$, a cost function $c$, and a target classifier $h$, a classifier $f^*$ is said to be at \emph{strategic local optimum} if
    \begin{align*}
         \Pr_{\vx \sim \calD}  [ h(\vx) = f^* (\Delta(\vx))] 
        = \!\! \max_{f \in \mathcal{N}(f^*)} \Pr_{\vx \sim \calD} [ h(\vx) = f(\Delta(\vx))].
    \end{align*}
\end{definition}

We proceed to state the main result of this section.

\begin{restatable}{theorem}{PLSHardness} \label{theorem:PLS_hardness}
    Given a finite population $X$, a distribution $\cal{D}$ over $X$, a cost function $c$, and a target classifier $h$, it is \PLS-hard to find a strategic local optimum as in \cref{def:strat_local_opt}. This result holds even when $c$ is a metric and the target classifier $h$ is provided explicitly to the algorithm.
\end{restatable}

\cref{theorem:PLS_hardness} shows that, unless there is a collapse in the complexity hierarchy (specifically, $\P = \PLS$), finding even a local optimum in strategic classification cannot be achieved in polynomial time.

The proof of \cref{theorem:PLS_hardness} takes an instance of \localmaxcut problem, which was shown to be \PLS-complete by \citet{schaffer1991simple}, and constructs a polynomial-time reduction to the problem of finding a strategic local optimum. 

In particular, let $G = (V, E, w)$ be a weighted undirected graph with edge weights $w_{(u, v)} \geq 0$ for any edge $(u, v) \in E.$ We construct an instance of strategic classification consisting of a finite population $X$ and a non-uniform distribution $\cal{D}$ over $X$. For each vertex $v \in V,$ we introduce a point $x_{v^-}$ with label $h(x_{v^-}) = 0.$ For each edge $(u, v) \in E$, we introduce two points $x_{(u, v)^+}$ and $x_{(u, v)^-}$ with labels $h(x_{(u, v)^+}) = 1$ and $h(x_{(u, v)^-}) = 0.$ We design the cost function $c$ with specific values such that the admissible deviations of the Contestant satisfy the following: the point $x_{(u, v)^-}$ may deviate only to $x_{(u, v)^+}$, and the vertex points $x_{u^-}$ and $x_{v^-}$ may deviate only to $x_{(u, v)^+}$. \cref{fig:PLS_gadget} illustrates the gadget for a simple graph with two vertices and one edge, where the dashed edges represent admissible deviations for the Contestant.

\begin{figure}
    \centering
    \scalebox{0.7}{\begin{tikzpicture}[
    node distance=2.5cm,
    font=\sffamily,
    vertex_point/.style={
        circle, 
        draw=blue!80!black, 
        fill=blue!10, 
        thick, 
        minimum size=1.2cm, 
        align=center
    },
    edge_pos_point/.style={
        rectangle, 
        draw=orange!80!black, 
        fill=orange!10, 
        thick, 
        minimum size=1.2cm, 
        rounded corners=3pt,
        align=center
    },
    edge_neg_point/.style={
        rectangle, 
        draw=orange!80!black, 
        fill=orange!10, 
        thick, 
        minimum size=1.2cm, 
        rounded corners=3pt,
        align=center
    },
    metric_link/.style={
        draw=gray!80, 
        line width=2pt, 
        dashed,
        shorten >=2pt, shorten <=2pt
    },
    strat_move/.style={
        ->, 
        color=red!80!orange, 
        line width=1.5pt, 
        >=latex, 
        dashed
    }
]

    \node[vertex_point] (u) at (-3, 3) {
        $x_{u^-}$ \\ 
        \footnotesize $h=0$
    };
    
    \node[vertex_point] (v) at (3, 3) {
        $x_{v^-}$ \\ 
        \footnotesize $h=0$
    };

    \node[edge_pos_point] (e_pos) at (0, 0) {
        $x_{(u, v)^+}$ \\ 
        \footnotesize $h=1$
    };

    \node[edge_neg_point] (e_neg) at (0, -3.5) {
        $x_{(u, v)^-}$ \\ 
        \footnotesize $h=0$
    };

    
    \node[above=0.2cm of u, text=black, font=\footnotesize] {Prob. $\propto \sum_{u' \in \mathcal{N}(u)} w_{(u, u')} $};
    \node[above=0.2cm of v, text=black, font=\footnotesize] {Prob. $\propto \sum_{v' \in \mathcal{N}(v) } w_{(v, v')}$};
    
    \node[right=0.2cm of e_pos, text=black, font=\footnotesize, align=left] {Prob. $\propto 2 w_{(u,v)}$};
    \node[left=0.2cm of e_pos, text=black, font=\bfseries\footnotesize] {$f(x_{(u, v)^+}) = 0$};

    \node[right=0.2cm of e_neg, text=black, font=\footnotesize, align=left] {Prob. $\propto 2w_{(u,v)} + 1$};
    \node[left=0.2cm of e_neg, text=black, font=\bfseries\footnotesize] {$f(x_{(u, v)^-}) = 0$};

    \draw[metric_link] (u) -- node[midway, fill=white, inner sep=1pt, text=black] {\small 0.8} (e_pos);
    \draw[metric_link] (v) -- node[midway, fill=white, inner sep=1pt, text=black] {\small 0.8} (e_pos);
    \draw[metric_link] (e_pos) -- node[midway, fill=white, inner sep=1pt, text=black] {\small 0.8} (e_neg);

    
    \draw[strat_move, bend left=15] (e_pos) to node[midway, below left, text=red!90!black, font=\bfseries\footnotesize, align=center] {Strategic move\\if $f(x_{u^-})=1$} (u);
    
    \draw[strat_move, bend right=60] (e_neg) to node[midway, right, xshift=2pt, text=red!90!black, font=\bfseries\footnotesize, align=left] {Strategic move\\if $f(x_{(u, v)^+})=1$} (e_pos);


\end{tikzpicture}}
    \caption{Our basic edge gadget for edge $(u, v)$. The Jury would like to classify $x_{(u, v)^+}$ as 1, but $x_{(u, v)^-}$ would then strategically deviate to $x_{(u, v)^+}$. This forces the Jury to pick a classifier such that $f(x_{(u, v)^+}) = 0 = f(x_{(u, v)^-})$. Furthermore, if the Jury switches the label of $x_{u^-}$ from $0$ to $1$, all edges incident to $u$ in the graph that were previously classified as $0$ can profitably deviate to $x_{u^-}$. The change in the Jury's utility reflects the change in the weight of the cut induced by $f$.}
    \label{fig:PLS_gadget}
\end{figure}
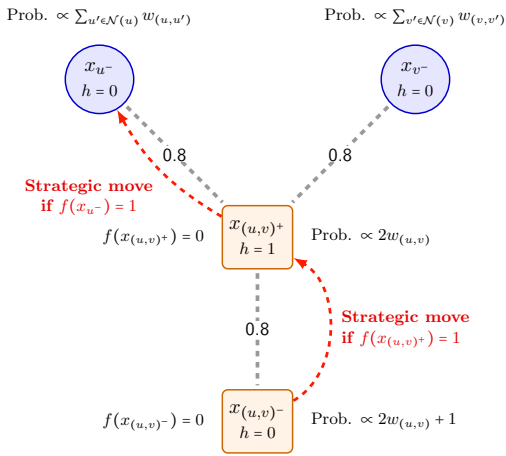

The first observation is that for any classifier $f^*$ that is at a strategic optimum, $f^*(x_{(u, v)^-}) = 0$ and $f^*(x_{(u, v)^+}) = 0.$ This follows from the fact that the distribution $\mathcal{D}$ assigns higher probability to $x_{(u, v)^-}$ than to $x_{(u, v)^+}.$ As a result, any classifier that labels either $x_{(u, v)^+}$ or $x_{(u, v)^-}$ as positive will incur a net loss due to misclassifying $x_{(u, v)^-}$. 

The result is that the Jury will only label vertex points $x_{v^-}$ as positive and improve their utility through the deviation of $x_{(u, v)^+}$ to the corresponding vertex points $x_{u^-}$ or $x_{v^-}.$ Through a careful design of the distribution weights, we show that any strategically local optimal classifier $f^*$ induces a cut of the original graph $G$: vertices $v \in V$ with $f^*(x_{v^-}) = 0$ lie on one side of the cut, while vertices $u \in V$ with $f^*(x_{u^-}) = 1$ lie on the other. Moreover, the strategic local optimality of $f^*$ ensures the local optimality of the induced cut. We defer the full proof to \cref{sec:append_PLS_proof}.

Notably, our construction in~\cref{theorem:PLS_hardness} based on \maxcut also yields an alternative proof that finding a global strategic optimum is \NP-complete~\citep{Hardt16:Strategic}.

\paragraph{Performative stability with endogenous costs} Moreover, we establish hardness for computing a Nash equilibrium---which translates to a performatively stable point---in an extension of~\Cref{def:strat-class} in which the Jury can also affect the cost incurred by deviating. This extension captures the fact that, in reality, decision-makers often act as regulators who go beyond merely classifying~\citep{Alhanouti25:Robust}; they actively dictate the cost structure by imposing sanctions and penalizing more certain deviation types, for example through audits~\citep{Estornell21:Incentivizing}. We refer to this setting as strategic classification with \emph{endogenous costs} (\Cref{def:strat-class-end}). Compared to~\Cref{def:strat-class}, our definition also posits that the Jury is constrained to select a classifier from a specified set of classifiers, and similarly for the Contestant. We find that this class of problems is rich enough to \PPAD-hard problems.

\begin{restatable}{proposition}{endogcosts}
    \label{prop:endogenous}
    Computing a performatively stable point in strategic classification with endogenous costs is \PPAD-hard.
\end{restatable}

The proof is based on a reduction from a win-loss two-player general-sum game $(\mat{A}, \mat{B})$ ~\citep{Abbott05:Complexity}. We associate each column to a separate classifier and each row to a point drawn uniformly from $X$. Each classifier $f_j$ labels the points in accordance with $\mat{B}_{:, j}$. The target classifier is taken to be $h(\vx) = 0$ for all $\vx \in X$. The Contestant can pick any constant deviation, and the payoff matrix $\mat{A}$ is encoded through the costs, making them the dominant component of the Contestant's utility. The detailed argument is deferred to~\Cref{sec:proofs}.

This reduction heavily relies on the presence of endogenous costs. The complexity of computing performatively stable points based on the more common~\Cref{def:strat-class} remains an open problem. Because of the particular payoff structure, we suspect that the latter problem may be easier.

\section{Conclusion}

We have established a sharp computational phase transition for performative stability, showing that while slight expansiveness is tractable, the problem quickly becomes \PPAD-hard. We also characterized the complexity of computing local strategic maxima in strategic classification. An important question that arises from our results is to characterize the complexity of computing performatively stable points in strategic classification per~\Cref{def:strat-class}. Another interesting avenue for future research is to close the gap between the expansiveness tolerance of our ellipsoid-based approach (\Cref{prop:exp-perf}), the recent result of~\citet{Diakonikolas25:Pushing}, and the \PPAD-hardness established in~\Cref{theorem:PPAD_complete}, thereby refining the complexity landscape that emerged from our paper (\Cref{fig:phasetransition}).

\section*{Acknowledgments}
Ioannis Panageas is supported by NSF grant CCF-
2454115. Tuomas Sandholm is supported by NIH award A240108S001, the Vannevar Bush Faculty Fellowship ONR N00014-23-1-2876, and National Science Foundation grant RI-2312342.

\bibliography{main}

\clearpage
\appendix

\section{Further related work}
\label{sec:furtherrelated}

This section highlights additional related research.

\paragraph{Multiagent performative prediction} While the original framework of~\citet{perdomo2021performativeprediction} considers a single decision-maker, many real-world applications involve multiple agents with different objectives. In such settings, the underlying distribution can depend on the joint deployment of all agents, giving rise to more complex distribution shifts. \citet{Piliouras22:Multiagent} formalized this as \emph{multi-agent performative prediction}, and showed that standard retraining dynamics can lead to complex behaviors, ranging from global stability to chaos. Performative prediction in multi-player settings has since received considerable attention~\citep{Narang23:Multiplayer,Gois25:Performative}. Although we draw upon techniques from algorithmic game theory, our results pertain to the single-agent setting.

\citet{Drusvyatskiy23:Stochastic} analyze performative stability through the lens of stochastic optimization with decision-dependent distributions, applying proximal-point methods and extra-gradient algorithms in performative prediction. From a broader standpoint, similar dynamic feedback issues are commonplace in reinforcement learning and robotics, where the act of learning changes the environment~\citep{Levine20:Offline}.

\paragraph{Complexity theory} $\PPAD$ was introduced by~\citet{Papadimitriou94:Complexity} and was famously shown to characterize the complexity of Nash equilibria in two-player general-sum games~\citep{Daskalakis09:The,Chen09:Settling}. Our \PPAD-hardness proof for general convex sets leverages certain tools developed by~\citet{rankgames} and \citet{ deligkas_et_al:LIPIcs.ICALP.2020.38}. As we highlighted earlier, much of the work in the complexity of fixed points and variational inequalities has focused on the case where the constraint set is the hypercube. A notable recent exception is the paper of~\citet{Attias25:Fixed}, which proves exponential lower bounds when the constraint set is the $\ell_2$ ball.

$\PLS$ was introduced by~\citet{schaffer1991simple} to characterize the complexity of (presumably) hard local optimization problems, such as max-cut under the so-called flip neighborhood. It is also known to characterize the complexity of \emph{pure} Nash equilibria in multi-player potential games~\citep{Fabrikant04:Complexity}. As a result, our results imply a polynomial-time equivalence between the complexity of local performative optimality and pure Nash equilibria in potential games.
\section{First-order performative stability}
\label{sec:firstorder}

In this section, we examine the notion of first-order approximate performative stability, which will be relevant for analyzing the computational complexity of finding performative stable points. We note that computing a performatively stable point can be viewed as a fixed point computation problem (\Cref{lemma:direc2}). Under mild assumptions on the induced distribution $\mathcal{D}(\cdot)$, if the objective function $\ell(\vx; \vz)$ is convex in $\vx$ for all $\vz$ and jointly continuous in $(\vx, \vz),$ a performatively stable point is guaranteed to exist \citep{perdomo2021performativeprediction}. However, when the objective function $\ell(\vx; \vz)$ is nonconvex in $\vx$, the $\argmin$  operator in \cref{def:perf-stab} may not be efficiently computable. To address this issue, we consider an alternative definition, which we restate below.

\begin{definition}[Performative stability] \label{def:local_ps}
A point $\vx^* \in \calX$ is (first-order) performatively stable if for all $\vx \in \calX$, it holds that
    \begin{equation*}   
        \left\langle\vx - \vx^*, \mathbb{E}_{\vz \sim \mathcal{D}(\vx^*)}\left[\nabla_{\vx} \ell(\vx^*; \vz)\right]\right\rangle \geq 0.
    \end{equation*}
\end{definition}

By virtue of existing complexity results pertaining to variational inequalities, computing a first-order performatively stable point lies in $\PPAD$ under mild assumptions on the representation of $\calX$, $\calD$, and $\ell$. This complements~\Cref{theorem:PPAD_complete}. 

\begin{corollary}
    \label{cor:PPAD-memb}
    Computing an $\epsilon$-performatively stable point is in \PPAD.
\end{corollary}

As we now show, when $\ell(\vx; \vz)$ is convex in $\vx$ for all $\vz$, the notion of first-order performative stability coincides with the notion of performative stability per~\Cref{def:perf-stab}.

    \begin{claim} \label{claim:first-order-stable-equivalence}
        If $\ell(\vx; \vz)$ is convex in $\vx$ for any $\vz$, $\vx^* \in \calX$ is a performatively stable point if and only if $\vx^*$ is a first-order performatively stable point.
    \end{claim}

    \begin{proof}
        Let $\vx^* \in \calX$ be a first-order performatively stable point, since $\ell(\vx; \vz)$ is convex in $\vx$, for any $\vx' \in \calX$, we have
        \begin{align}
            \mathbb{E}_{\vz \sim \mathcal{D}(\vx^*)}\left[\ell(\vx'; \vz)\right] - \mathbb{E}_{\vz \sim \mathcal{D}(\vx^*)}\left[\ell(\vx^*; \vz)\right] \notag & \geq \mathbb{E}_{\vz \sim \mathcal{D}(\vx^*)}\left[\left\langle\vx' - \vx^*, \nabla_{\vx}\ell(\vx^*; \vz) \right\rangle\right] \notag \\ & = \left\langle \vx' - \vx^*, \mathbb{E}_{\vz \sim \mathcal{D}(\vx^*)}\left[\nabla_{\vx} \ell(\vx^*; \vz)\right] \right\rangle \\
            & \geq 0, \label{eq:plug_in_local_stable_def}
        \end{align}
        where in \eqref{eq:plug_in_local_stable_def} we use the definition of first-order performative stability. On the other hand, suppose $\vx^* \in \calX$ is a performatively stable point, for any $\vx' \in \calX$ and any $\alpha \in (0, 1]$, it holds that
        \begin{align*}
            \mathbb{E}_{\vz \sim \mathcal{D}(\vx^*)}\left[\ell(\vx^*; \vz)\right] \leq \mathbb{E}_{\vz \sim \mathcal{D}(\vx^*)}\left[\ell(\vx^* + \alpha (\vx' -\vx^*); \vz)\right].
        \end{align*}
        By Taylor's theorem
        \begin{align*}
            \mathbb{E}_{\vz \sim \mathcal{D}(\vx^*)}\left[\ell(\vx^* + \alpha (\vx' -\vx^*); \vz)\right] - \mathbb{E}_{\vz \sim \mathcal{D}(\vx^*)}\left[\ell(\vx^*; \vz)\right] = \alpha \left\langle\vx' - \vx^*, \mathbb{E}_{\vz \sim \mathcal{D}(\vx^*)}\left[\nabla_{\vx} \ell(\vx^*; \vz)\right]\right\rangle + o(\alpha).
        \end{align*}
        If $\left\langle\vx' - \vx, \mathbb{E}_{\vz \sim \mathcal{D}(\vx^*)}\left[\nabla_{\vx} \ell(\vx^*; \vz)\right]\right\rangle < 0$, it would imply that for a sufficiently small $\alpha > 0$, we have $\mathbb{E}_{\vz \sim \mathcal{D}(\vx^*)}\left[\ell(\vx^* + \alpha (\vx' -\vx^*); \vz)\right] - \mathbb{E}_{\vz \sim \mathcal{D}(\vx^*)}\left[\ell(\vx^*; \vz)\right] < 0$, contradicting the fact that $\vx^*$ is a performatively stable point per \cref{def:perf-stab}. The proof is complete.
    \end{proof}

\section{Contraction for general norms}
\label{appendix:generalnorms}

This section generalizes the contraction proof of~\citet{perdomo2021performativeprediction} from the $\|\cdot\|_2$ norm to arbitrary norms. In particular, we adapt~\Cref{ass:l2} as follows.

\begin{assumption}
    \label{ass:norms}
    Let $\ell(\vx; \vz)$ be the loss function and $\calD(\vx)$ the distribution on $\calZ$ induced by $\vx \in \calX$.
    \begin{itemize}
        \item (strong convexity) $\ell(\vx; \vz)$ is $\alpha$-strongly convex with respect to $\|\cdot\|$:
        \begin{equation*}
            \ell(\vx; \vz) \ge \ell(\vx' ; \vz) + \langle \nabla_{\vx} \ell(\vx'; \vz), \vx - \vx' \rangle + \frac{\alpha}{2} \|\vx - \vx' \|^2
        \end{equation*}
        for any $\vx, \vx' \in \calX$ and $\vz \in \calZ$.
        \item (smoothness) $\ell(\vx; \vz)$ is $\beta$(-jointly) smooth if
        \begin{equation*}
            \|\nabla_{\vx} \ell(\vx; \vz) - \nabla_\vx \ell(\vx'; \vz) \|_* \leq \beta \|\vx - \vx' \|
        \end{equation*}
        and
        \begin{equation*}
            \|\nabla_{\vx} \ell(\vx; \vz) - \nabla_\vx \ell(\vx; \vz') \|_* \leq \beta \|\vz - \vz' \|
        \end{equation*}
        for any $\vx, \vx' \in \calX$ and $\vz, \vz' \in \calZ$.
        \item (sensitivity) $\calD$ is $L$-sensitive if
        \begin{equation*}
            W_1(\calD(\vx), \calD(\vx')) \leq L \|\vx - \vx' \|
        \end{equation*}
        for any $\vx, \vx' \in \calX$, where $W_1$ denotes the Wasserstein-1 distance, or earth mover's distance.
    \end{itemize}
\end{assumption}

Above, we denote by $\|\cdot\|_*$ the dual norm of $\|\cdot\|$. We point out that the contraction argument of~\citet{perdomo2021performativeprediction} readily carries over under~\Cref{ass:norms}.

\contractgen*

\begin{proof}
Let $\vx, \vx' \in \calX$, $f(\vy) = \E_{\vz \sim \calD(\vx) } \ell(\vy; \vz)$, and $f'(\vy) = \E_{\vz \sim \calD(\vx') } \ell(\vy; \vz)$. Taking the expectation over $\vz \sim \calD(\vx)$, it follows that $f(\vy)$ is $\alpha$-strongly convex with respect to $\|\cdot\|$. Thus,
\begin{equation}
    \label{eq:strong-convex1}
    f(G(\vx)) \geq f(G(\vx')) + \langle G(\vx) - G(\vx'), \nabla f(G(\vx')) \rangle + \frac{\alpha}{2} \| G(\vx) - G(\vx') \|^2,
\end{equation}
where $G$ is the RRM mapping (\Cref{def:RRM}). Since $G(\vx)$ is, by definition, the unique minimizer of $f$, we also have $\langle G(\vx') - G(\vx), \nabla f(G(\vx)) \rangle \geq 0$ by the first-order optimality condition. In turn, this implies
\begin{equation}
    \label{eq:strong-convex2}
    f(G(\vx')) \geq f(G(\vx)) + \frac{\alpha}{2} \| G(\vx) - G(\vx') \|^2.
\end{equation}
Combining~\eqref{eq:strong-convex1} and~\eqref{eq:strong-convex2}, we have
\begin{equation}
    \label{eq:convexbound}
    \langle G(\vx') - G(\vx), \nabla f(G(\vx')) \rangle \geq \alpha \|G(\vx) - G(\vx') \|^2.
\end{equation}
Furthermore, $\langle G(\vx') - G(\vx), \nabla \ell(G(\vx'); \vz) \rangle$ is $( \| G(\vx') - G(\vx) \| \beta )$-Lipschitz continuous in $\vz$ since
\begin{align*}
    &| \langle G(\vx') - G(\vx), \nabla \ell(G(\vx'); \vz) \rangle - \langle G(\vx') - G(\vx), \nabla \ell(G(\vx'); \vz') \rangle | \\
    &\hspace{18em} \leq \|G(\vx') - G(\vx) \| \|\nabla \ell(G(\vx'); \vz) - \nabla \ell(G(\vx'); \vz') \|_* \\
    &\hspace{18em} \leq \beta \|G(\vx') - G(\vx) \|,
\end{align*}
by $\beta$-joint smoothness. Now, for the distribution map $\calD(\cdot)$, Kantorovich-Rubinstein duality yields
\begin{equation*}
    \left| \E_{\vz \sim \calD(\vx)} g(\vz) -  \E_{Z \sim \calD(\vx')} g(\vz) \right| \leq L \|\vx - \vx' \| \quad \forall g \text{ 1-Lipschitz}.
\end{equation*}
As a result,
\begin{equation*}
    \langle G(\vx) - G(\vx'), \nabla f(G(\vx')) \rangle - \langle G(\vx) - G(\vx'), \nabla f'(G(\vx')) \rangle \geq - L \beta \|G(\vx') - G(\vx) \| \|\vx - \vx' \|.
\end{equation*}
By the first-order optimality condition, it also follows that $\langle G(\vx) - G(\vx'), \nabla f'(G(\vx')) \rangle \geq 0$ since $G(\vx')$ is the minimizer of $f'$. So,
\begin{equation}
    \label{eq:Wass-bound}
    \langle G(\vx) - G(\vx'), \nabla f(G(\vx')) \rangle \geq - L \beta \|G(\vx') - G(\vx) \| \|\vx - \vx' \|.
\end{equation}
Combining~\eqref{eq:convexbound} and~\eqref{eq:Wass-bound}, we conclude that
\begin{equation*}
    L \beta \|G(\vx') - G(\vx) \| \|\vx - \vx' \| \geq \alpha \|G(\vx) - G(\vx') \|^2 \implies \|G(\vx) - G(\vx') \| \leq \frac{L \beta}{\alpha} \|\vx - \vx' \|.
\end{equation*}
In other words, if $L \beta / \alpha < 1 $, $G$ is a contraction with respect to the norm $\|\cdot\|$, as claimed.
\end{proof}

Similar extensions are possible for other algorithms beyond repeated risk minimization, such as repeated gradient descent.
\section{Ellipsoid for Euclidean expansive mappings}
\label{sec:ellipsoid}

A well-known result in optimization is that there is a polynomial-time algorithm for computing fixed points of nonexpansive mappings \emph{with respect to the $\ell_2$ norm}~\citep{Huang99:Approximating,Sikorski93:Ellipsoid}; the complexity of this problem is a major open question for more general norms. In particular, for mappings that are contracting with respect to the $\ell_2$ norm, there is an algorithm whose complexity does not depend on the contraction parameter. In the setting of performative prediction, we begin by observing that this can be used to obtain a significant improvement in the setting where $\nicefrac{\beta L}{\alpha} \approx 1$. The number of iterations needed to reach an approximate fixed point under repeated risk minimization is proportional to $\log(1/\epsilon) \frac{1}{ \frac{\alpha}{\beta L} - 1 }$ in the regime where $\nicefrac{\beta L}{\alpha} \approx 1$, thereby blowing up. 

\begin{theorem}[\citealp{Sikorski93:Ellipsoid}]
    \label{theorem:ellipsoid}
 Consider a continuous mapping $T : \calX \to \calX$, where $\calX$ is a subset of the $d$-dimensional Euclidean space, that is nonexpansive with respect to the $\ell_2$ norm; that is, $\|T(\vx) - T(\vx') \|_2 \leq \|\vx - \vx' \|_2$ for any $\vx, \vx' \in \calX$. Then there is a $\poly(d, \log(1/\epsilon))$-time algorithm that computes an $\epsilon$-fixed point of $T$.
\end{theorem}

We provide the simple proof below, as we will use a similar bound in our extension.

\begin{proof}[Proof of~\Cref{theorem:ellipsoid}]
    We will prove that the operator $F : \vx - T(\vx)$ is monotone. That is, $\langle F(\vx) - F(\vx'), \vx - \vx' \rangle \geq 0$ for any $\vx, \vx' \in \calX$. Indeed, we write
    \begin{equation}
        \label{eq:monotone}
        \langle F(\vx) - F(\vx'), \vx - \vx' \rangle = \frac{1}{2} \left( \|\vx - \vx' \|_2^2 - \| T(\vx) - T(\vx') \|_2^2 + \| \vx - T(\vx) - \vx' + T(\vx') \|_2^2 \right) \geq 0
    \end{equation}
    since $\|\vx - \vx' \|_2 \geq \|T(\vx) - T(\vx') \|_2$ and $\|\cdot\|_2 \geq 0$. Now, let $\vx$ be an $\epsilon$-approximate VI solution with respect to $F$, which means that $\langle \vx' - \vx, F(\vx) \rangle \geq - \epsilon$ for any $\vx' \in \calX$. In particular, setting $\vx' = T(\vx)$ yields $\langle T(\vx) - \vx, \vx - T(\vx) \rangle \geq - \epsilon$, or $- \| \vx - T(\vx) \|^2_2 \geq - \epsilon$, which is to say that $\vx$ is a $\sqrt{\epsilon}$-fixed point of $T$. Moreover, an $\epsilon$-approximate VI solution with respect to $F$ can be computed in time $\poly(d, \log(1/\epsilon))$ since $F$ is monotone. This completes the proof.
\end{proof}

There is also a more direct argument that does not go through the monotonicity of the gap function. In particular, it is possible to develop a separation oracle by relying on the fact that $T$ is nonexpansive: for any point $\vx_k \in \calX$, first test whether $T(\vx_k) = \vx_k$. If not, the key observation is that $\vec{g}_k = \vx_k - T(\vx_k)$ serves as a separating hyperplane. This is so because $\langle \vx_k - \vx, \vec{g}_k \rangle \geq 0$ for any $\vx \in \calX$ that is a fixed point of $T$; since $\langle \vx_k - \vx, \vec{g}_k \rangle = \langle \vx_k - \vx, F(\vx_k) \rangle = \langle \vx_k - \vx, F(\vx_k) - F(\vx) \rangle \geq 0$ by~\eqref{eq:monotone}.

\begin{corollary}
    If $\contr = L \beta / \alpha \leq 1$ (per~\Cref{ass:l2}), there is a $\poly(d, \log(1/\epsilon))$-time algorithm for computing an $\epsilon$-performatively stable point.
\end{corollary}

To put this into better context, it is important to point out that repeated risk minimization can fail when $\contr = 1$. For completeness, we include the simple example below.

\begin{example}[Cycling dynamics at the threshold]
    \label{example:cycling}
    Consider a one-dimensional setting where $\calX$ is centrally symmetric and the loss is the simple quadratic objective $\ell(\vx; \vz) = \frac{1}{2}\|\vx - \vz\|^2$. This function is $1$-jointly smooth and $1$-strongly convex ($\beta = \alpha = 1$). Suppose further that the distribution $\calD(\vx)$ is a point mass supported on $\vz = g(\vx) \defeq -\vx$. The sensitivity of this map is $L=1$, resulting in $\contr = L \beta/\alpha = 1$.
    
    The repeated risk minimization (RRM) update at step $t$ minimizes the loss on the distribution induced by the current iterate $\vx_t$. Since the distribution is supported on $\vz = -\vx_t$, the update becomes:
    \begin{align}
        \vx_{t+1} = \argmin_{\vx \in \calX} \frac{1}{2} \|\vx - (-\vx_t)\|^2 = -\vx_t.
    \end{align}
    Starting from any initialization $\vx_0 \neq 0$, the algorithm oscillates indefinitely between $\vx_0$ and $-\vx_0$, failing to converge to the unique performatively stable point $\vxstar = 0$.
\end{example}

\paragraph{Extension to expansive mappings} Interestingly, we observe that~\Cref{theorem:ellipsoid} can be extended when $T$ can be marginally expansive. Let us first present an approach that works for monotone operators, and we shall then relax the monotonicity assumption. We rely on the notion of an \emph{expected variational inequality (EVI)}~\citep{Zhang25:Expected}. In particular, an $\epsilon$-EVI solution $\mu \in \Delta(\calX)$ satisfies
\begin{equation}
    \label{eq:EVI}
    \E_{\vx \sim \mu} [\langle F(\vx), \vx - \vx' \rangle] \leq \epsilon \quad \forall \vx' \in \calX.
\end{equation}
\citet{Zhang25:Expected} gave a $\poly(d, \log(1/\epsilon))$ for computing an $\epsilon$-EVI solution. We will first argue that, for monotone operators, the mean of the distribution $\bar{\vx} = \E_{\vx \sim \mu}[\vx]$ is an $\epsilon$-approximate solution to the \emph{Minty} VI problem:
\begin{equation}
    \label{eq:MVI}
    \langle F(\vx'), \vx' - \bar{\vx} \rangle \geq -\epsilon \quad \forall \vx' \in \calX.
\end{equation}
Indeed, starting from~\eqref{eq:EVI} and using monotonicity, we have that for any $\vx' \in \calX$,
\begin{align*}
    \epsilon &\geq \E_{\vx \sim \mu} [\langle F(\vx), \vx - \vx' \rangle] \\
    &\geq \E_{\vx \sim \mu} [\langle F(\vx'), \vx - \vx' \rangle] \\
    &= \langle F(\vx'), \E_{\vx \sim \mu}[\vx] - \vx' \rangle \\
    &= \langle F(\vx'), \bar{\vx} - \vx' \rangle.
\end{align*}

Rearranging, this establishes~\eqref{eq:MVI}. Finally, to go from an $\epsilon$-MVI solution to an approximate VI solution, we use the following standard lemma.
\begin{lemma}[Relation between $\epsilon$-MVI and SVI]
    \label{lemma:MVI-SVI}
    Let $F: \calX \to \mathbb{R}^d$ be an operator that is $L$-Lipschitz continuous. If $\vxstar \in \calX$ is an $\epsilon$-approximate MVI solution, then $\vxstar$ is an $O_\epsilon(\sqrt{\epsilon})$-approximate (Stampacchia) VI solution. Specifically, if $D$ is the $\ell_2$ diameter of $\calX$,
    \begin{equation}
        \langle F(\vxstar), \vx - \vxstar \rangle \geq -2 D \sqrt{L \epsilon} \quad \forall \vx \in \calX.
    \end{equation}
\end{lemma}

\begin{proof}
    Let $\vx \in \calX$ be an arbitrary target point. For any $\delta \in (0, 1]$, we define the interpolation point $\vx' = \vxstar + \delta(\vx - \vxstar) \in \calX$. Using the fact that $\vxstar$ is an $\epsilon$-MVI solution,
    \begin{equation*}
        \delta \langle F(\vx'), \vx - \vxstar \rangle \geq -\epsilon \implies \langle F(\vx'), \vx - \vxstar \rangle \geq -\frac{\epsilon}{\delta}.
    \end{equation*}
    We now relate $F(\vx')$ to $F(\vxstar)$ using the Lipschitz continuity of $F$:
    \begin{align*}
        \langle F(\vxstar), \vx - \vxstar \rangle &= \langle F(\vx'), \vx - \vxstar \rangle + \langle F(\vxstar) - F(\vx'), \vx - \vxstar \rangle \\
        &\geq -\frac{\epsilon}{\delta} - \|F(\vxstar) - F(\vx')\|_2 \|\vx - \vxstar\|_2 \\
        &\geq -\frac{\epsilon}{\delta} - L \|\vxstar - \vx'\|_2 \|\vx - \vxstar\|_2 \\
        &= -\frac{\epsilon}{\delta} - L \delta \|\vx - \vxstar\|_2^2 \\
        &\geq -\frac{\epsilon}{\delta} - L \delta D^2.
    \end{align*}
    The claim follows by picking $\delta$ optimally.
\end{proof}

We now extend this approach under \emph{hypomonotonicity}~\citep{Iusem03:Inexact,Alber05:Regularization,Alomar24:Hypomonotone}. In particular, a mapping $F$ satisfies $\sigma$-hypomonotonicity for $\sigma > 0$ if
\begin{equation}
    \langle F(\vx) - F(\vx'), \vx - \vx' \rangle \geq -\sigma \|\vx - \vx'\|^2
\end{equation}
for all $\vx, \vx' \in \calX$. Starting again from~\eqref{eq:EVI}, we have that for any $\vx' \in \calX$,
\begin{align*}
    \epsilon &\geq \E_{\vx \sim \mu} \left[ \langle F(\vx'), \vx - \vx' \rangle - \sigma \|\vx - \vx'\|^2 \right] \\
    &= \langle F(\vx'), \bar{\vx} - \vx' \rangle - \sigma \E_{\vx \sim \mu} [\|\vx - \vx'\|^2].
\end{align*}
As a result,
\begin{equation*}
    \langle F(\vx'), \vx' - \bar{\vx} \rangle \geq - \epsilon - \sigma D^2 \quad \forall \vx' \in \calX.
\end{equation*}
This means that $\bar{\vx}$ is an $(\epsilon + \sigma D^2)$-approximate MVI solution. Combining with~\Cref{lemma:MVI-SVI}, we have shown the following.

\begin{proposition}
    \label{prop:hypomonotone}
    Let $F : \calX \to \R^d$ be a $\rho$-hypomonotone $L$-Lipschitz continuous operator. There is a $\poly(d, \log(1/\epsilon))$-time algorithm for computing an $(\epsilon+ \sigma D^2)$-approximate MVI solution, which is in turn a $2 D \sqrt{L ( \epsilon + \sigma D^2 ) }$-approximate VI solution.
\end{proposition}
We now show how to translate this result for finding fixed points of a slightly expansive mapping $T$: $\| T(\vx) - T(\vx') \|_2 \leq (1 + \sigma) \|\vx - \vx' \|_2$. As in~\eqref{eq:monotone}, if $F(\vx) = \vx - T(\vx)$, we have
\begin{align*}
        \langle F(\vx) - F(\vx'), \vx - \vx' \rangle &= \frac{1}{2} \left( \|\vx - \vx' \|_2^2 - \| T(\vx) - T(\vx') \|_2^2 + \| \vx - T(\vx) - \vx' + T(\vx') \|_2^2 \right) \\
        &\geq - \left( \sigma + \frac{\sigma^2}{2} \right) \|\vx - \vx' \|_2^2
\end{align*}
for any $\vx, \vx' \in \calX$. In other words, $F$ is $(\sigma + \frac{\sigma^2}{2})$-hypomonotone. Furthermore, if $\vxstar$ is an $\epsilon'$-VI solution for $F$, it follows that $\|T(\vxstar) - \vxstar \|_2 \leq \sqrt{\epsilon'}$. We arrive at the following conclusion.

\begin{proposition}
    \label{prop:gen-expans}
    Let $T : \calX \to \calX$ be a such that $\|T(\vx) - T(\vx') \|_2 \leq (1 + \sigma) \|\vx - \vx' \|_2$. There is a $\poly(d, \log(1/\epsilon))$-time algorithm for computing an $\epsilon'$-fixed point of $T$, where
    \begin{equation*}
        \epsilon' = \sqrt{ 2 D \sqrt{ (2 + \sigma) \left(\epsilon + \left( \sigma + \frac{\sigma^2}{2} \right) D^2 \right)  } }.
    \end{equation*}
    In particular, if $\sigma \leq \epsilon$, $\epsilon' = \Theta_\epsilon(\epsilon^{1/4})$.
\end{proposition}

Compared to the recent result of~\citet{Diakonikolas25:Pushing}, the complexity above grows logarithmically in $1/\epsilon$, at the cost of being applicable to a narrower regime of $\rho$. Furthermore, as we highlighted in~\Cref{sec:complexity}, \Cref{prop:hypomonotone} yields a polynomial-time algorithm for computing $\epsilon$-performatively stable points in the following regime.

\expans*

Whether the tradeoff between approximation and expansiveness can be improved to match the result of~\citet{Diakonikolas25:Pushing} is an interesting question. As becomes evident from~\Cref{prop:hypomonotone,prop:gen-expans}, the $\epsilon^{1/4}$ factor is an artifact of how approximation is measured. In terms of the VI problem corresponding to $F(\vx) = \vx - T(\vx)$, our approach yields an $O_\epsilon(\epsilon)$ approximation for a Minty VI solution and an $O_\epsilon(\sqrt{\epsilon})$ approximation for a (Stampacchia) VI solution.

\section{\PPAD-hardness for general convex sets}\label{sec:append_convex_set_PPAD}

    


In this section, we generalize the result of \cref{theorem:PPAD_complete} from the domain $[0, 1]^d$ to general convex sets. Through out the section, we let $n$ denote the bit-length of the input to the Turing machine. We start this section by defining the \TwoDSperner problem. Consider the triangle $\triangle A_1A_2A_3$ on a 2D plane where $A_0 = (0, 0), A_1 = (2^n, 0),$ and $A_2 = (0, 2^n).$ We define the triangulation to be 
\begin{equation*}
    T_n = \{\vp = (p_1, p_2)\in \mathbb{Z}^2 \mid p_1 \geq 0, p_2 \geq 0, p_1 + p_2 \leq 2^n\}.
\end{equation*}
For any $3$-coloring function $g :T_n \to \{1, 2, 3\},$ it is said to be admissible if the following conditions are met:
\begin{itemize}
    \item $g(A_i) = i, \text{ for all $i \in \{1, 2, 3\}$};$
    \item For every $\vp$ on the segment of $A_iA_j$, $g(\vp) \neq 6 - i - j.$
\end{itemize}
    
\begin{definition}[\TwoDSperner; \citealp{PapaGraphTheory}]
    Given a polynomial-time Turing machine $F$ that produces a admissible 3-coloring $g$ on $T_n$ where $g(\vp) = F(\vp) \in \{1, 2, 3\}$ for every $\vp \in T_n$, the output of \TwoDSperner is a trichromatic triangle of coloring $g$.
\end{definition}

The \PPAD-membership of \TwoDSperner was established by~\citet{PapaGraphTheory}, \citet{CHEN20094448} showed that \TwoDSperner is \PPAD-complete.

\begin{theorem}[\citealp{CHEN20094448}] \TwoDSperner is \PPAD-complete. \label{theorem:2dPPAD}
\end{theorem}

We note that even though the \PPAD-hardness result for \TwoDSperner is established on a right triangle, one can generalize this hardness result to arbitrary triangles.

\begin{lemma} \label{lemma:triangulation}
    For any triangle $\triangle {A}_1 {A}_2 {A}_3$ where ${A}_1 = (0, 0), {A}_2 = (a_1, a_2) = {\va}, {A}_3 = (b_1, b_2) = {\vb}$, define the triangulation to be 
    \begin{equation*}
        \mathcal{T}_n = \left\{\vp = \frac{q}{2^n}\va + \frac{r}{2^n}\vb \mid (q, r) \in \mathbb{Z}^2, q \geq 0, r \geq 0, q + r \leq 2^n \right\}. 
    \end{equation*}
    Given a polynomial-time Turing machine $F'$ that produces an admissible 3-coloring $g'$ for all points $\vp \in \mathcal{T}_n,$ it is \PPAD-complete to output a trichromatic triangle of coloring $g'$. 
\end{lemma}

\begin{proof}
     First, observe that given $\vp, \va$ and $\vb,$ we can compute coefficients $q$ and $r$ in polynomial time through standard basis decomposition. The \PPAD-membership follows from Sperner's lemma. To prove the hardness, given a \TwoDSperner instance $(F, 0^n)$, we construct the coloring $g'$ of triangle $\triangle A_1 A_2 A_3$ such that for any point $\vp = (\frac{q}{2^n}\va, \frac{r}{2^n}\vb) \in \mathcal{T}_n$
    \begin{equation*}
        g'(\vp) = F((q, r)).
    \end{equation*}
    Since $F$ produces an admissible 3-coloring, it holds that 
    \begin{itemize}
        \item $g'(A_i) = i, \text{ for all $i \in \{1, 2, 3\}$};$
        \item For every $\vp = (\frac{q}{2^n} \va, \frac{r}{2^n} \vb)$ on the segment of $A_i A_j, g'(\vp) = g((q, r)) \neq 6 - i - j$.
    \end{itemize}
    Thus we show that $g'$ is an admissible 3-coloring over the triangulation $\mathcal{T}_n$. Furthermore, from any trichromatic triangle of coloring $g'$ over $\mathcal{T}_n$, we can recover a trichromatic triangle of coloring $g$ in $T_n$ in polynomial time. From \cref{theorem:2dPPAD}, we conclude the problem is \PPAD-complete.
\end{proof}

We now introduce the problem of \epsilonThickBrou \cite{deligkas_et_al:LIPIcs.ICALP.2020.38}, which is a extension of \TwoDSperner on an arbitrary triangle $\triangle A_1 A_2 A_3$ such that the coloring $g(\cdot)$ satisfies the following boundary conditions:

For a given $\epsilon$ and any $\vp = \frac{q}{2^n} \va + \frac{r}{2^n} \vb \in \mathcal{T}_n$, it holds that

\begin{equation} \label{eq:thick_brouwer_rule}
    g(\vp) = \begin{cases}
        1 & \text{for all $q \leq 2^n \epsilon$, and for all $ 2^n \epsilon < r < (1 - \epsilon)2^n - q;$}\\
        2 & \text{for all $r \leq 2^n \epsilon $, and for all $q < (1 - \epsilon)2^n - r$;}\\
        3 & \text{for all $q$ and $r$ such that $(1 - \epsilon) 2^n \leq q + r \leq 2^n;$} \\
        \text{any color in $\{1, 2, 3\}$} & \text{otherwise}. 
    \end{cases}
\end{equation}


Given a \TwoDSperner instance, one can reduce it to \epsilonThickBrou in polynomial time by increasing the number of points in the triangle and embedding the original instance in the center of the new construction. A detailed proof can be found in the paper of~\citet{deligkas_et_al:LIPIcs.ICALP.2020.38}.

To map the coloring defined on the grid $\mathcal{T}_n$ to the triangle $\triangle A_1 A_2 A_3$, we adopt the bit-extraction technique, which is commonly used in \PPAD-reductions. Specifically, consider a triangle $\triangle A_1 A_2 A_3$ with vertices ${A}_1 = \mathbf{0},$ ${A}_2 = \va$, and ${A}_3 = \vb.$ For any point $\vp$ inside the triangle $\triangle A_1 A_2 A_3$, we can compute coefficients $q$ and $r$ through standard basis decomposition such that
\begin{equation*}
    \vp = \frac{q}{2^n} \va + \frac{r}{2^n} \vb.
\end{equation*}

\begin{algorithm}
    \caption{ExtractBit $(x, b)$ \label{algo:extract_bits}}
    \begin{algorithmic}
        \STATE $b \gets 0.5$
        \STATE $b \gets x -^b b$
        \STATE $b \gets b *^b L$
    \end{algorithmic}
\end{algorithm}

We then apply the bit-extraction scheme \cref{algo:extract_bits} of \citet{deligkas_et_al:LIPIcs.ICALP.2020.38} to recover the first $n$ bits of $\frac{q}{2^n}$ and $\frac{r}{2^n}$. Operators $+^b, -^b, *^b$ denote bounded operations that ensure the outcomes remain in $[0, 1],$ which can be efficiently implemented through a algorithmic circuit with standard $\min$ and $\max$ operations. Notice that when $x \leq 0.5$, we have $b = 0$, and when $x \geq 0.5 + \frac{1}{L},$ we have $b = 1.$ For $0.5 < x < 0.5 + \frac{1}{L},$ the value of $b$ may lie anywhere strictly between $0$ and $1$ due to the continuity of the output of the algorithmic circuit. We refer to the first two cases as well-positioned and the last case as poorly-positioned. To account for the effect of poorly-positioned points, for any $\vx \in [0, 1]^2$, we sample $k$ points $\vx_1, \vx_2, \dots, \vx_k$ where
\begin{equation*}
    \vx_i = \vx + (i - 1) \left[\frac{1}{(k + 1)2^{n+1}}, \frac{1}{(k + 1)2^{n+1}}\right].
\end{equation*}
The following lemma holds for the sample points $\vx_1, \dots, \vx_k$.

\begin{lemma}[\citealp{deligkas_et_al:LIPIcs.ICALP.2020.38}] \label{lemma:sample_points}
    Setting $L = (k + 2) 2^{n+1}$, then among points $\vx_1, \dots, \vx_k$, at most two points will be poorly-positioned.
\end{lemma}

We proceed to restate the main results of this section.

\convexsethardness*

\begin{proof}
    Let $\mathcal{X} \subset \mathbb{R}^2$ be a two-dimensional well-bounded domain. From \cref{def:well-bounded}, there exist a 2D ball $\mathcal{B}_{R_1}$ inside $\mathcal{X}.$ Consider an arbitrary triangle equilateral triangle $\triangle A_1 A_2 A_3$ that lies on the boundary of $\mathcal{B}_{R_1}$. Without loss of generality, we assume $R_1 = 1$ and further assume the position of $A_1$ is at $(0, 0),$ $A_2 = (\sqrt{3}, 0) = \va$, and $A _3 = (\frac{\sqrt{3}}{2}, \frac{3}{2}) = \vb.$ Let the discretized grid over $\triangle A_1 A_2 A_3$ be as defined in \cref{lemma:triangulation}. We set $\epsilon = \frac{1}{8}$ and assign colors to the grid points according to~\eqref{eq:thick_brouwer_rule} such that the coloring $g(\cdot)$ for points on the grid $\mathcal{T}_n$ is admissible for the \epsilonThickBrou problem. For point $\vx \in \cal{X}$ outside $\triangle A_1 A_2 A_3$, the coloring $g(\vx)$ is defined as
    \begin{equation} \label{eq:outside_coloring}
        g(\vx) = \begin{cases}
        1 \text{ } \text{if } \min \{ \text{dist}(\vx ,{A}_1 {A}_2), \text{dist}(\vx ,{A}_1 {A}_3),  \text{dist}(\vx ,{A}_2 {A}_3) \} = \text{dist}(\vx ,{A}_1 {A}_3); \\
            2 \text{ } \text{if } \min \{ \text{dist}(\vx ,{A}_1 {A}_2), \text{dist}(\vx ,{A}_1 {A}_3),  \text{dist}(\vx ,{A}_2 {A}_3) \} = \text{dist}(\vx ,{A}_1 {A}_2); \\
            3 \text{ } \text{if } \min \{ \text{dist}(\vx ,{A}_1 {A}_2), \text{dist}(\vx ,{A}_1 {A}_3),  \text{dist}(\vx ,{A}_2 {A}_3) \} = \text{dist}(\vx ,{A}_2 {A}_3); \\
            \text{In terms of ties:\hspace{-1cm}} \\
        1 \text{ } \text{if } \min \{ \text{dist}(\vx ,{A}_1 {A}_2), \text{dist}(\vx ,{A}_1 {A}_3),  \text{dist}(\vx ,{A}_2 {A}_3) \} = \text{dist}(\vx ,{A}_1 {A}_2) =  \text{dist}(\vx ,{A}_1 {A}_3); \\
        2 \text{ } \text{if } \min \{ \text{dist}(\vx ,{A}_1 {A}_2), \text{dist}(\vx ,{A}_1 {A}_3),  \text{dist}(\vx ,{A}_2 {A}_3) \} = \text{dist}(\vx ,{A}_1 {A}_2) =  \text{dist}(\vx ,{A}_2 {A}_3); \\
        3 \text{ } \text{if } \min \{ \text{dist}(\vx ,{A}_1 {A}_2), \text{dist}(\vx ,{A}_1 {A}_3),  \text{dist}(\vx ,{A}_2 {A}_3) \} = \text{dist}(\vx ,{A}_1 {A}_3) =  \text{dist}(\vx ,{A}_2 {A}_3),
        \end{cases}
    \end{equation}
    where dist$(\vx, A_i A_j)$ denotes the distance from point $\vx$ to line $A_i {A}_j$. Notice that by construction, there is no trichromatic triangle outside $\triangle {A}_1 {A}_2 {A}_3.$

    We move on to map the coloring of $\vx$ to the operator value $F(\vx).$ Recall that segment $A_1 A_2 = \va$ and segment $A_1 {A}_3 = \vb,$ we define $\va_{\perp} = (0, 1)$ to be the unit vector orthogonal to $\va$ and pointing inside the triangle. Similarly, we define $\vb_{\perp} = (\frac{\sqrt{3}}{2}, -\frac{1}{2})$ and $\vc_{\perp} = (-\frac{\sqrt{3}}{2}, -\frac{1}{2})$ as the unit vectors orthogonal to segment ${A}_1 {A}_3$ and ${A}_2 {A}_3$ respectively, each pointing inward, as shown in \cref{fig:PPAD_color}. Since $\triangle A_1 A_2 A_3$ is a equilateral triangle, it follows that $\va_{\perp} + \vb_{\perp} + \vc_{\perp} = 0.$ We then map each color to a different vector such that color $1$ is mapped to $\vb_\perp$, color $2$ is mapped to $\va_\perp$, and color $3$ is mapped to $\vc_\perp$. For any point $\vx \in \cal{X}$, we first sample $k$ points\footnote{The coloring for any sample point $\vx_i$ outside $\cal{X}$ is also determined by \eqref{eq:outside_coloring}.} such that
    \begin{equation*}
        \vx_i = \vx + (i - 1) \left[\frac{1}{(k + 1)2^{n+1}} \va, \frac{1}{(k + 1)2^{n+1}} \vb\right].
    \end{equation*}
    We then extract the first $n$ bits of $\vx_i$ as $\bar{\vx}_i$ and pass to the boolean circuit to get the corresponding color. The operator $F(\vx)$ is then computed as the average of the vectors corresponding to the colors of sampled points,
    \begin{equation*}
        F(\vx) = \frac{1}{k} \left(\sum_{i = 1}^k \bbm{1}(g(\bar{\vx}_i) = 2)\va_\perp + \bbm{1}(g(\bar{\vx}_i) = 1)\vb_\perp + \bbm{1}(g(\bar{\vx}_i) = 3)\vc_\perp\right),
    \end{equation*}
    where $\bbm{1}(g(\bar{\vx}_i) = j)$ denotes the indicator that the coloring of $\bar{\vx}_i$ is $j$. 
    

    What remains now is to show that from a solution $\vx^*$ of \eqref{eq:target_VI}, one can recover a trichromatic triangle. We begin by showing that any point within $\epsilon$ distance from the boundary of $\cal{X}$ cannot be a solution of \eqref{eq:target_VI}. 
    
    If $\vx$ lies within $\frac{\epsilon}{2}$ distance of the boundary of $\cal{X},$ it either lies outside of $\triangle {A}_1 {A}_2 {A}_3$ or is within $\frac{\epsilon}{2}$ distance from one of the segment ${A}_i {A}_j.$
    By the coloring rule in \eqref{eq:thick_brouwer_rule} and \eqref{eq:outside_coloring}, among the sampled points $\vx_1, \dots \vx_k$, any well-positioned point can only take two of the three colors. Without loss of generality, we assume that color $2$ is missing from all well-positioned sample points of $\vx$, the cases where color $1$ or color $3$ is missing follow similarly.

    From~\cref{lemma:sample_points},  at least $k - 2$ points out of the $k$ sample points are well-conditioned and are assigned either color $1$ or $3$. We consider two cases,
    \begin{itemize}
        \item \textbf{Color $1$ is also missing among the well-conditioned sample points}. By the coloring rule in \eqref{eq:thick_brouwer_rule} and the the choice of $\epsilon$, it follows that $\vx$ must lie within distance $\frac{\epsilon}{2}$ of segment $A_2 A_3$. Recall that $\vc_\perp = (-\frac{\sqrt{3}}{2}, -\frac{1}{2}),$ we consider the $x$-coordinate of $F(\vx),$
        \begin{align*}
            F(\vx)_x \leq - (\frac{\sqrt{3}(k-2)}{2k} - \frac{2}{k}),
        \end{align*}
        where the first term comes from the contribution of the $k-2$ well-conditioned sample points, while the second term accounts for the error introduced by the remaining two points. Let $\vx' = A_1 = (0, 0),$ we have
        \begin{align*}
            \langle\vx' - \vx, F(\vx)\rangle & \geq (\vx'_x - \vx_x) F(\vx)_x \\
            & \geq (\frac{\sqrt{3}}{2}  - \frac{\epsilon}{2}) \cdot (\frac{\sqrt{3}(k-2)}{2k} - \frac{2}{k}) \\
            & \geq \epsilon',
        \end{align*}
        where the second inequality holds because $\vx$ is within $\epsilon$ distance from $A_2 A_3$ segment, so its $x$-coordinate, $\vx_x \geq \frac{\sqrt{3}}{2} - \epsilon.$ The third inequality holds by setting $k \geq 16$ and $\epsilon' \leq \frac{\epsilon}{8} = \frac{1}{32}.$
        \item \textbf{The well-conditioned sample points contain both color $1$ and color $3$}. From the coloring rule in \eqref{eq:thick_brouwer_rule}, $\vx$ cannot lie within $\frac{\epsilon}{2}$ distance with segment $A_1 A_2$. Recall that $\vb_{\perp} = (\frac{\sqrt{3}}{2}, -\frac{1}{2})$ and $\vc_{\perp} = (-\frac{\sqrt{3}}{2}, -\frac{1}{2})$. Consequently, the averaged direction $F(\vx)$ also has negative $y$-component. Specifically, let $F(\vx)_y$ denote the $y$-coordinate of $F(\vx)$, it holds that
    \begin{align*}
        F(\vx)_y \leq - \left(\frac{k - 2}{2k}  - \frac{2}{k}\right).
    \end{align*}
    Let $\vx'$ be a point on the segment $A_1 A_2$. Note that since color $1$ is missing, by the construction of the \epsilonThickBrou problem, we also have $\vx_y - \vx'_y \geq \frac{\epsilon}{2}.$ Therefore
    \begin{align*}
        \left\langle \vx' - \vx, F(\vx)\right\rangle
        & \geq (\vx'_y - \vx_y) F(\vx)_y\\
        & \geq \frac{\epsilon}{2} \cdot \left(\frac{k - 2}{2k}  - \frac{2}{k}\right) \\
        & > \epsilon',
    \end{align*}
    where the last step follows by setting $k \geq 16$ and $\epsilon' \leq \frac{\epsilon}{8} = \frac{1}{32}.$
    \end{itemize}

    For any point $\vx$ that lies more than $\frac{\epsilon}{2}$ distance away from the boundary, we argue that if $\vx$ is a solution for $\eqref{eq:target_VI},$ then one can recover a trichromatic triangle for the \epsilonThickBrou problem. We show this by contradiction, first assume that if color $1$ is missing from the well-positioned sampled points among $\vx_1 \cdots \vx_k,$ like the previous case, we have
    \begin{align*}
        F(\vx)_y \leq - \left(\frac{k - 2}{2k} - \frac{2}{k}\right).
    \end{align*}
    Note that since $\vx$ is not within $\frac{\epsilon}{2}$ distance from the boundary, along the negative $y$ direction, we can find a point $\vx' \in \cal{X}$ that is at least $\epsilon$ away from $\vx$ (i.e., $\vx'$ = $\vx - \epsilon \cdot (0, 1)$). It then holds that
    \begin{align*}
        \left\langle \vx' - \vx, F(\vx)\right\rangle
        & > \frac{\epsilon}{2} \cdot \left(\frac{k - 2}{2k} - \frac{2}{k}\right) > \epsilon'.
    \end{align*}
    The cases where color $2$ or coloring $3$ is missing follow similarly. Therefore we conclude that if $\vx$ is a solution of \eqref{eq:target_VI}, the well-positioned sample points among $\vx_1, \dots, \vx_n$ must have all three colors. Observe that $\norm{\vx_k - \vx}_\infty < \frac{1}{2^n} \min(\norm{\va}_2, \norm{\vb}_2)$, which implies that if the well-positioned points among $\vx_1 \dots \vx_k$ contain all three colors, then $\vx$ must resides within a trichromatic square with sides oriented along directions $\va$ and $\vb.$ Such a square can only occur within $\triangle A_1 A_2 A_3.$ Finally, the reduction from trichromatic triangles to trichromatic squares of \TwoDSperner is established in \citet{CHEN20094448}. 
    
    Note that since the grid $\mathcal{T}_n$ has side length $O(\frac{1}{2^n}),$ the Lipschitz constant of the operator $F(\cdot)$ is $O(2^n).$ Define the rescaled operator $F'(\vx) = \frac{F(\vx)}{2^n}$, and let $\epsilon '' = \frac{\epsilon'}{2^n} = O(\frac{1}{2^n}).$ Then computing a point $\vx^* \in \mathcal{X}$ such that for any $\vx \in \cal{X},$ 
    \begin{align*}
        \langle\vx - \vx^*, F'(\vx^*)\rangle \leq \epsilon''        
    \end{align*}
    is \PPAD-hard. Moreover, the Lipschitz constant for $F'(\cdot)$ is $O(1)$ and $\epsilon'' = O(\frac{1}{2^n})$. This completes the proof.
\end{proof}

\begin{remark}
    In our proof the operator $F$ is given by a (well-behaved) arithmetic circuit with $n$ rational inputs and size that depends polynomially on the description of the $\TwoDSperner$ problem, which can effectively approximate any Lipschitz continuous function.  We refer the reader to~\citet{Fearnley23:Complexity} for further background on complexity theory.
\end{remark}

\begin{remark}
    We remark that our construction uses irrational coordinates for the positions of $A_2$ $A_3$ and for the directional vectors $\va_{\perp}, \vb_{\perp},\ \vc_{\perp}$, which cannot be represented exactly by a Turing machine. Nevertheless, our reduction continues to work given a suitably good approximation of these quantities. A similar technical issue is discussed in \citet{deligkas_et_al:LIPIcs.ICALP.2020.38}.
\end{remark}
\section{\PLS-hardness of finding local optima in strategic classification} \label{sec:append_PLS_proof}

In this section, we establish that finding a local optima in strategic classification per \cref{def:strat-class} is \PLS-hard. We first restate the main result we want to prove.

\PLSHardness*

\begin{proof}
     The proof proceeds via a polynomial-time reduction from the \localmaxcut problem. Let $G = (V, E, w)$ be a weighted undirected graph with edge weights $w_{(u, v)} \geq 0$ for any edge $(u, v) \in E.$ We construct a strategic classification instance with a finite population $X$ and a non-uniform distribution $\cal{D}$ over $X$. For convenience, we define the \emph{weight} of a point $w_{\calD}(x)$ so that the probability of sampling $x \in X$ from $\cal{D}$ is proportional to $w_{\calD}(x)$. The population is defined as follows.
     \begin{itemize}
         \item For each vertex $v \in V,$ we introduce a point $x_{v^-}$ with label $h(x_{v^-}) = 0.$ The weight of $x_{v^-}$ under distribution $\cal{D}$ is given by the total weight of edges incident to vertex $v$, i.e., $w_{\cal{D}}(x_{v^-}) = \sum_{u \in \mathcal{N}(v)} w_{(u, v)};$
         \item For every edge $(u, v) \in E$, we introduce a point $x_{(u, v)^+}$ with label $h(x_{(u, v)^+}) = 1$, and weight $w_{\cal{D}}(x_{(u, v)^+}) = 2 w_{(u, v)};$
         \item For every edge $(u, v) \in E$, we introduce a point $x_{(u, v)^-}$ with label $h(x_{(u, v)^-}) = 0$, and weight $w_{\cal{D}}(x_{(u, v)^-}) = 2 w_{(u, v)} + 1.$
     \end{itemize}
    We now define a metric $c: X \times X \to \mathbb{R}_{\geq 0}.$ We choose the value such that $c(x, x) = 0$ and $c(x, y) = c(y, x).$ Moreover, the metric $c$ takes only two nonzero values, $0.8$ and $1.2$. These values are chosen to ensure that the triangle inequality holds, other than that, we can set them to arbitrary value in the range from $(0, 1)$ and $(1, \infty)$ respectively. The metric is defined as follows.
    \begin{itemize}
        \item For each vertex point $x_{v^-}$, and for every positive edge point $x_{(u, v)^+}$ such that edge $(u, v)$ is incident to vertex $v$, we set $c(x_{v^-}, x_{(u, v)^+}) = 0.8$;
        \item For each edge positive point $x_{(u, v)^+}$ and corresponding edge negative point $x_{(u, v)^-}$ we set $c(x_{(u, v)^+}, x_{(u, v)^-}) = 0.8;$
        \item For all other pairs $(x, y),$ we define $c(x, y) = 1.2.$
    \end{itemize} 
    If $c(x, y) = 0.8,$ we call them close to each other. Notice that under this metric $c$, the Contestant will only deviate a point $x$ to point $y$ if $f(x) = 0, f(y) = 1,$ and $x$ is close to $y.$

    The first claim is that if $f^*$ is at a strategic local optimum, then $f^*(x_{(u, v)^-}) = 0$ and $f^*(x_{(u, v)^+}) = 0$ for all edge $(u, v) \in E$. To see this, assume that $f^*(x_{(u, v)^-}) = 1$ for some edge $(u, v) \in E,$ then the Jury can simply deviate to another classifier $f'$ that differs with $f^*$ with only the prediction of $x_{(u, v)^-}$. If $M$ is the sum of all the weights of points in $X$, it holds that
    \begin{align}
        \Pr_{\vx \sim \calD} [ h(\vx) = f'(\Delta(\vx))] - \Pr_{\vx \sim \calD} [ h(\vx) = f^*(\Delta(\vx))] & \geq \frac{1}{M} \left(2 w_{(u, v)} + 1 - 2 w_{(u, v)}\right) \label{eq:local_M_improve_ineq}\\
        & = \frac{1}{M} \label{eq:local_M_improve},
    \end{align}
    where~\eqref{eq:local_M_improve_ineq} holds because changing the prediction of $x_{(u, v)^-}$ from $1$ to $0$ may cause the misclassification of the positive edge point $x_{(u, v)^+}$, but it ensures that the negative edge point $x_{(u, v)^-}$ is classified correctly. As a result, \eqref{eq:local_M_improve} implies that the Jury can strictly improve their utility by deviating to the classifier $f'$, which contradicts the assumption that $f^*$ is a strategic local optimum.

    Now suppose $f^*(x_{(u, v)^+}) = 1$ for some $(u, v) \in E.$ Since $f^*(x_{(u, v)^-}) = 0$ and $x_{(u, v)^+}$ and $x_{(u, v)^-}$ are close to each other, the Contestant will deviate $x_{(u, v)^-}$ to $x_{(u, v)^+}.$ Consider an alternative classifier $f'$ that differs from $f^*$ only in the prediction of $x_{(u, v)^+},$ we have
    \begin{align*}
        \Pr_{\vx \sim \calD} [ h(\vx) = f'(\Delta(\vx))] - \Pr_{\vx \sim \calD} [ h(\vx) = f^*(\Delta(\vx))] & \geq \frac{1}{M} \left(- 2 w_{(u, v)} - ( - 2 w_{(u, v)} - 1)\right) \\
        & = \frac{1}{M},
    \end{align*}
    where the first inequality holds because changing the prediction of $x_{(u, v)^+}$ from $1$ to $0$ may cause the positive edge point $x_{(u, v)^+}$ to be misclassified, but it ensures that the negative edge point $x_{(u, v)^-}$ is classified correctly. Moreover, any deviations of vertex points $x_{v^-}$ can only increase this gap. Thus, $f^*$ cannot be a strategic local optimum, yielding a contradiction.

    We conclude that, in order to reach a strategic local optimum, the only points that can be labeled positively are the vertex points $x_{v^-}$. In this case, the Jury can improve their utility through the deviation of positive edge points $x_{(u, v)^+}$ to the corresponding vertex points $x_{v^-}.$ 
    
    We now proceed to analyze the utility of the Jury when the label of a single vertex point $x_{v^-}$ is changed from $0$ to $1$. Let $\mathcal{N} (v)^+$ denote the set of neighbors of vertex $v$ in the original graph whose corresponding vertex points are labeled $1$ by the classifier, and let $\mathcal{N} (v)^-$ denote the set of neighbors that are labeled $0$. Observe that before changing the label of $x_{v^-}$, for all vertices $u \in \mathcal{N}(v)^+,$ the Contestant already deviates the corresponding positive edge points $x_{(u, v)^+}$ to $x_{u^-}$. Hence, those points will be correctly labeled regardless of the change. In contrast, for every $u \in \mathcal{N}(v)^-$, the corresponding positive edge points $x_{(u, v)^+}$ do not deviate before the change, but will deviate to $x_{v^-}$ after the change. 
    
    Let $f$ denote the Jury's classifier before the change and $f'$ the classifier after the change. The resulting difference in the Jury's utility is
    \begin{align}
         \Pr_{\vx \sim \calD} [ h(\vx) = f'(\Delta(\vx))] - \Pr_{\vx \sim \calD} [ h(\vx) = f(\Delta(\vx))] & = \frac{1}{M} \left(\sum_{u \in \mathcal{N}(v)^-} 2 w_{(u, v)} - \sum_{u' \in \mathcal{N}(v)}  w_{(u', v)}\right) \label{eq:PLS_first_case}\\
         & =  \frac{1}{M} \left(\sum_{u \in \mathcal{N}(v)^-} w_{(u, v)} - \sum_{u' \in \mathcal{N}(v)^+} w_{(u', v)}\right) \label{eq:PLS_first_gain}.
    \end{align}
    The first term in \eqref{eq:PLS_first_case} corresponds to the gain from the correctly labeling positive edge points after the change, while the second term accounts for the loss introduced by misclassifying the vertex point $x_{v^-}.$ 

    Similarly, consider the case where the Jury change the label of a vertex point $x_{v^-}$ from $1$ to $0$. For each vertex $u \in \mathcal{N}(v)^+,$ the corresponding edge points $x_{(u, v)^+}$ will still be classified positive since the Contestant will deviate to $x_{u^-}.$ However, for every vertex $u \in \mathcal{N}(v)^-$, edge points $x_{(u, v)^+}$ will be misclassified, since after the change they no longer have any positively labeled neighbors. Thus, the resulting change in the Jury's utility is
    
    \begin{align}
         \Pr_{\vx \sim \calD} [ h(\vx) = f'(\Delta(\vx))] - \Pr_{\vx \sim \calD} [ h(\vx) = f(\Delta(\vx))] & = \frac{1}{M} \left( - \sum_{u \in \mathcal{N}(v)^-} 2 w_{(u, v)} + \sum_{u' \in \mathcal{N}(v)}  w_{(u', v)}\right) \label{eq:PLS_second_case}\\
         & =  \frac{1}{M} \left(\sum_{u \in \mathcal{N}(v)^+} w_{(u, v)} - \sum_{u' \in \mathcal{N}(v)^-} w_{(u', v)}\right) \label{eq:PLS_second_gain}.
    \end{align}

    Here, the first term in \eqref{eq:PLS_second_case} captures the loss from misclassifying edge points $x_{(u, v)^+},$ while the second term is due to correctly labeling the vertex point $x_{v^-}.$

    Suppose we have a classifier $f^*$ which is at a strategic local optimum. By \cref{def:strat_local_opt}, \eqref{eq:PLS_first_gain} is nonpositive for every vertex $u$ such that $x_{u^-}$ is labeled $0$, and \eqref{eq:PLS_second_gain} is nonpositive for every vertex $v$ such that $x_{v^-}$ is labeled $1$. Now consider a cut of the original graph defined as follows: each vertex $v \in V$ with $f^*(x_{v^-}) = 0$ is on one side of the cut (negative side), and each vertex with $f^*(x_{v^-}) = 1$ is placed on the other side (positive side). Since \eqref{eq:PLS_first_gain} is nonpositive, moving any vertex from the negative side of the cut to the positive side cannot increase the total weight of the cut. Similarly, since \eqref{eq:PLS_second_gain} is nonpositive, moving any vertex from the positive side to the negative side also cannot improve the cut weight. Thus, we conclude that any locally strategic optimal classifier $f^*$ induces a local max cut on the original graph $G$. This completes the proof.
 \end{proof}
\section{Further omitted proofs}
\label{sec:proofs}

This section contains additional omitted proofs. We begin with~\Cref{theorem:PPAD_complete}.

\PPADaffine*

\begin{proof}
    Let $\calX = [0, 1]^d$ and let $\vxstar \in \calX$ be an $\epsilon$-performatively stable point of \eqref{eq:strongly-convex-objective}-\eqref{eq:g-def}. By~\cref{def:approx-perfstab}, we have that for all $\vx \in \calX$,
    \begin{align}
        \left\langle\vx - \vxstar, \vxstar - g(\vxstar)\right\rangle \geq -\epsilon. \label{eq:PPAD_VI}
    \end{align}
    Now, let $g(\vx) : (\mat{I} - \bar{\mat{A}})\vx - \bar{\vb}$, where $\bar{\mat{A}} = \frac{\epsilon}{\epsilon'} \mat{A}$ and $\bar{\vb} = \frac{\epsilon}{\epsilon'} \vb$ for $\mat{A}$ and $\vb$ as in \cref{lemma:inapproximation}. Finding a solution $\vxstar$ satisfying \eqref{eq:PPAD_VI} would imply that for all $\vx \in [0, 1]^d,$
    \begin{align}
        \left\langle\vx - \vxstar, \mat{A}\vxstar + \vb\right\rangle \geq - \epsilon'.
    \end{align}
    From \cref{lemma:inapproximation}, we conclude that it is \PPAD-complete to find a point $\vxstar$ satisfying~\eqref{eq:PPAD_VI}. Furthermore,
    \begin{align*}
        \norm{g(\vx) - g(\vx')} &= \norm{(\mat{I} - \bar{\mat{A}})\vx - (\mat{I} - \bar{\mat{A}}) \vx'} \\
        & \leq (\norm{\mat{I}} + \norm{\mat{\bar{A}}}_2) \norm{\vx - \vx'} \\
        & \leq (1 + \sqrt{\norm{\bar{\mat{A}}}_1\norm{\bar{\mat{A}}}_\infty})\norm{\vx - \vx'} \\
        & \leq (1 + \frac{\epsilon}{\epsilon'}) \norm{\vx - \vx'}.
    \end{align*}
    We conclude that even when $L \beta / \alpha \leq 1 + \frac{\epsilon}{\epsilon'}$, it is \PPAD-complete to find an $\epsilon$-performatively stable point.
\end{proof}

We next point out the polynomial equivalence between the two natural ways of measuring approximation for performatively stable points.

\begin{lemma}
    \label{lemma:direc1}
If $\ell(\vx; \vz)$ is $\alpha$-strongly convex in $\vx$ (with respect to the $\|\cdot\|_2$ norm) for any $\vz \in \calZ$ and $\vxstar$ is an $\epsilon$-performatively stable point (\Cref{def:approx-perfstab}), then
\begin{equation*}
    \|\vxstar - G(\vxstar)\|_2 \leq \sqrt{\frac{\epsilon}{\alpha}},
\end{equation*}
where $G$ is the RRM map (\Cref{def:RRM}).
\end{lemma}

\begin{proof}
Since $\ell(\vx; \vz)$ is $\alpha$-strongly convex in $\vx$ for any $\vz$, the expected loss $f(\vx) = \mathbb{E}_{\vz \sim \calD(\vxstar)}[\ell(\vx; \vz)]$ is also $\alpha$-strongly convex. By strong convexity, we have for any $\vx, \vx' \in \calX$,
\begin{equation}
    \label{eq:another-strong-conv}
    \langle \nabla f(\vx) - \nabla f(\vx'), \vx - \vx' \rangle \geq \alpha \|\vx - \vx'\|_2^2.
\end{equation}
We now write
\begin{equation*}
    \langle \nabla f(\vxstar) - \nabla f(G(\vxstar)), \vxstar - G(\vxstar) \rangle = \langle \nabla f(\vxstar), \vxstar - G(\vxstar) \rangle - \langle \nabla f(G(\vxstar)), \vxstar - G(\vxstar) \rangle.
\end{equation*}
We bound each term separately. First, since $\vxstar$ is an $\epsilon$-performatively stable point, we have
\begin{equation*}
    \langle \vx - \vxstar, \nabla f(\vxstar) \rangle \geq -\epsilon \quad \forall \vx \in \calX.
\end{equation*}
Taking $\vx = G(\vxstar)$, we get $\langle \nabla f(\vxstar), \vxstar - G(\vxstar) \rangle \leq \epsilon$. Second, since $G(\vxstar) = \argmin_{\vx \in \calX} f(\vx)$, the first-order optimality condition yields $\langle \vx - G(\vxstar), \nabla f(G(\vxstar)) \rangle \geq 0$ for all $\vx \in \calX$, which in turn implies $-\langle \nabla f(G(\vxstar)), \vxstar - G(\vxstar) \rangle \leq 0$. Combining these bounds with~\eqref{eq:another-strong-conv},
\begin{align*}
    \alpha \|\vxstar - G(\vxstar)\|_2^2 &\leq \langle \nabla f(\vxstar) - \nabla f(G(\vxstar)), \vxstar - G(\vxstar) \rangle \leq \epsilon,
\end{align*}
and the proof follows.
\end{proof}

\begin{lemma}
    \label{lemma:direc2}
If $\ell(\vx; \vz)$ satisfies $\|\nabla_{\vx} \ell(\vx; \vz) - \nabla_{\vx} \ell(\vx'; \vz)\|_2 \leq \beta \|\vx - \vx'\|_2$, any point $\vxstar \in \calX$ such that $\|\vxstar - G(\vxstar)\|_2 \leq \epsilon$, where $G$ is the RRM map (\Cref{def:RRM}) is $\epsilon'$-performatively stable (\Cref{def:approx-perfstab}) with $\epsilon' = \epsilon \left( D \beta + \|\nabla f(G(\vxstar)) \|_2\right)$, where $D$ is the $\ell_2$ diameter of $\calX$.
\end{lemma}

\begin{proof}
Let $f(\vx) = \mathbb{E}_{\vz \sim \calD(\vxstar)}[\ell(\vx; \vz)]$. We have
\begin{align}
    \label{align:goal}
    \langle \vx - \vxstar, \nabla f(\vxstar) \rangle &= \langle \vx - \vxstar, \nabla f(G(\vxstar)) \rangle + \langle \vx - \vxstar, \nabla f(\vxstar) - \nabla f(G(\vxstar)) \rangle.
\end{align}
For the first term, we write
\begin{align*}
    \langle \vx - \vxstar, \nabla f(G(\vxstar)) \rangle &= \langle \vx - G(\vxstar), \nabla f(G(\vxstar)) \rangle + \langle G(\vxstar) - \vxstar, \nabla f(G(\vxstar)) \rangle \\
    &\geq \langle G(\vxstar) - \vxstar, \nabla f(G(\vxstar)) \rangle,
\end{align*}
by the first-order optimality condition of $G(\vxstar)$. Thus,
\begin{equation*}
    \langle \vx - \vxstar, \nabla f(G(\vxstar)) \rangle \geq - \|\vxstar - G(\vxstar)\|_2 \|\nabla f(G(\vxstar))\|_2 \geq -\epsilon \|\nabla f(G(\vxstar))\|_2.
\end{equation*}
For the second term in the right-hand side of~\eqref{align:goal}, we use $\beta$-Lipschitz continuity of $\nabla f$ to get
\begin{equation*}
    \langle \vx - \vxstar, \nabla f(\vxstar) - \nabla f(G(\vxstar)) \rangle \geq - \|\vx - \vxstar\|_2 \|\nabla f(\vxstar) - \nabla f(G(\vxstar))\|_2 \geq - \epsilon D \beta,
\end{equation*}
and the proof follows.
\end{proof}

To conclude, we provide the proof of~\Cref{prop:endogenous}. We begin by stating a variation of~\Cref{def:strat-class} that incorporates classifier-dependent costs. It also forces the Jury to select a classifier from a specified set, and similarly for the Contestant.

\begin{definition}[Strategic classification with endogenous costs]
    \label{def:strat-class-end}
    Strategic classification is a game played between the \emph{Jury} and the \emph{Contestant}. Let $\calD$ be a distribution over a population $X$, $c : X \times X \to \R_{\geq 0}$ a cost function, and $h$ a target classifier.
    \begin{enumerate}
        \item The Jury first publishes a classifier $f_j : X \to \{0, 1\}$ selected from a set of classifiers $\{f_1, \dots, f_m \}$.
        \item The Contestant selects a deviation $\Delta_i : X \to X$ selected from a set of deviations $\{\Delta_1, \dots, \Delta_n \}$.
    \end{enumerate}
    The payoff to the Jury is $\Pr_{\vx \sim \calD} [ h(\vx) = f_j(\Delta_i(\vx))]$ and the payoff to the Contestant is $\E_{\vx \sim \calD} [ f_j(\Delta_i(\vx)) - c_j(\vx, \Delta_i(\vx))]$. 
\end{definition}

While the Jury has now to decide among a small set of possible classifiers (as opposed to $2^{|X|}$), we show that computing a performatively stable point is \PPAD-hard.

\endogcosts*

Our reduction makes use of the hardness of Nash equilibria in two-player games.

\begin{definition}[Nash equilibrium]
    For a two-player game $(\mat{A}, \mat{B})$, with $\mat{A}, \mat{B} \in \R^{n \times m}$, an \emph{$\epsilon$-Nash equilibrium} is a point $(\vx, \vy) \in \Delta^n \times \Delta^m$ such that
    \[
        \langle \vx, \mat{A} \vy \rangle \geq \langle \hvx, \mat{A} \vy \rangle - \epsilon \text{ and } \langle \vx, \mat{B} \vy \rangle \geq \langle \vx, \mat{B} \hvy \rangle - \epsilon \quad \forall (\hvx, \hvy) \in \Delta^n \times \Delta^m.
    \]
\end{definition}

\begin{proof}[Proof of~\Cref{prop:endogenous}]
We reduce from the \PPAD-hard problem of computing a Nash equilibrium of a win-loss game~\citep{Abbott05:Complexity}. Let $(\mat{A}, \mat{B}) \in \{0, 1\}^{n \times m}$ be the payoff matrices for the row player and column player, respectively. We construct an instance of strategic classification with endogenous costs as follows. The population domain $X = \{\vx_1, \dots, \vx_n\}$ comprises $n$ distinct points. The underlying distribution $\calD$ is assumed to be uniform over $X$. The target classifier is $h(\vx) = 0$ for all $\vx \in X$.

The Jury chooses a classifier from the set $\{f_1, \dots, f_m\}$. We associate each classifier $f_j$ with the $j$th column of the game matrices. We define the classifier's outputs to have the opposite label from the column player's utility matrix: $f_j(\vx_i) = 1 - \mat{B}_{ij}$ for all $\vx_i \in X$. Moreover, because costs are endogenous, the Jury's choice of strategy $j$ also induces a specific cost function $c_j$.

The Contestant chooses a deviation from the set $\{\Delta_1, \dots, \Delta_n\}$. We restrict these to be \emph{constant deviations}, where $\Delta_i$ maps every input point to the specific point $\vx_i$ corresponding to the $i$th row. Now, the payoff to the Jury is $\Pr_{\vx \sim \calD} [ h(\vx) = f(\Delta(\vx))]$. As a result, under a classifier $f_j$ and a deviation $\Delta_i$, the utility of the Jury reads 
\[
    \Pr_{\vx \sim \calD} [ h(\vx) = f_j(\Delta_i(\vx))] = \Pr_{\vx \sim \calD} [ f_j(\Delta_i(\vx)) = 0] = \bbm{1} [ f_j(\vx_i) = 0] = \mat{B}_{i j},
\]
so this matches the utility of the column player in the original game. To ensure the Contestant (approximately) maximizes $\mat{A}$, we consider the following star metric for each $j$. Each point $\vx_i$ is connected to a point $\vxstar$. The cost to go from $\vx_i$ to $\vxstar$ is defined as $2 M - M \mat{A}_{i j}$ for a large parameter $M \gg 1$. Thus, $c_j( \vx_{i}, \vx_{i'} ) = c_{j}(\vx_i, \vxstar) + c_{j}(\vx_{i'}, \vxstar)$.

Under a classifier $f_j$ and a deviation $\Delta_i$, the payoff to the Contestant is 
\begin{equation*}
    \E_{\vx \sim \calD} [f_j(\Delta_i(\vx)) - c_j(\vx, \Delta_i(\vx))] = f_j(\vx_i) - \frac{1}{n} \sum_{i'=1}^n c_j(\vx_i, \vx_{i'}) = f_j(\vx_i) - \frac{1}{n} \sum_{i'=1}^n c_j(\vx_{i'}, \vxstar) - c_j(\vx_{i}, \vxstar).
\end{equation*}
The second term above does not depend on the deviation $\Delta_i$, so it is strategically irrelevant. Specifically, we end up with the two-player game with utilities $\langle \vx, \mat{A}' \vy \rangle + \langle \vec{c}, \vy \rangle$ and $\langle \vx, \mat{B} \vy \rangle$, where $\vec{c} = ( - \frac{1}{n} \sum_{i'=1}^n c_j(\vx_{i'}, \vxstar) )_{j=1}^m $, and $\mat{A}' = \mat{F} + M \mat{A}_{i j} - 2 M \mat{1}$ for $\mat{F}_{i j} = f_j(\vx_i)$; $\mat{1}$ denotes the all-ones matrix. Let $(\vx, \vy) \in \Delta^n \times \Delta^m$ be Nash equilibrium of this game, which corresponds to a perfomatively stable point of the strategic classification instance. We have $\langle \vx, \mat{A}' \vy \rangle \geq \langle \hvx, \mat{A}' \vy \rangle$ for any $\hvx \in \Delta^n$, which implies $\langle \vx, \mat{F} \vy \rangle + M \langle \vx, \mat{A} \vy \rangle \geq \langle \hvx, \mat{F} \vy \rangle + M \langle \hvx, \mat{A} \vy \rangle $. Since $\mat{F}_{i j} \in \{0, 1\}$, it follows that $\langle \vx, \mat{A} \vy \rangle \geq \langle \hvx, \mat{A} \vy \rangle - \frac{1}{M}$ for any $\hvx \in \Delta^n$. As a result, $(\vx, \vy)$ is a $1/M$-Nash equilibrium of the original two-player game.
\end{proof}

\end{document}